\documentclass[11pt]{article}

\usepackage{color,mathrsfs}
\usepackage{graphicx,subfigure,amsmath,amssymb,amsfonts,bm,epsfig,epsf,url,dsfont}
\usepackage{amsthm}
\usepackage{tikz}
\usepackage{bbm}      
\usepackage{booktabs}
\usepackage{cases}
\usepackage{fullpage}
\usepackage[small,bf]{caption}
\usepackage[top=1in,bottom=1in,left=1in,right=1in]{geometry}
\usepackage{fancybox}
\usepackage{verbatim}

\usepackage{algpseudocode} 
\usepackage{float}

\usepackage[linesnumbered,algoruled,boxed,lined]{algorithm2e}

\SetKwInOut{Parameter}{Parameters}

\usepackage{mathtools}

\usepackage{scalerel,stackengine}
\stackMath
\newcommand\reallywidehat[1]{%
\savestack{\tmpbox}{\stretchto{%
  \scaleto{%
    \scalerel*[\widthof{\ensuremath{#1}}]{\kern-.6pt\bigwedge\kern-.6pt}%
    {\rule[-\textheight/2]{1ex}{\textheight}}
  }{\textheight}%
}{0.5ex}}%
\stackon[1pt]{#1}{\tmpbox}%
}

\newcommand{\bE}{\mathbb{E}}

\newcommand{\Var}{\text{Var}}

\newcommand{\bI}{\mathbf{1}}

\newtheorem{problem}{Problem}

\newenvironment{fminipage}%
  {\begin{Sbox}\begin{minipage}}%
  {\end{minipage}\end{Sbox}\fbox{\TheSbox}}

\newcommand{\BlackBox}{\rule{1.5ex}{1.5ex}}  
\ifdefined\proof
    \renewenvironment{proof}{\par\noindent{\bf Proof\ }}{\hfill\BlackBox\\[2mm]}
\else
    \newenvironment{proof}{\par\noindent{\bf Proof\ }}{\hfill\BlackBox\\[2mm]}
\fi
 
\newtheorem{theorem}{Theorem}
\newtheorem{lemma}[theorem]{Lemma} 
 
\newtheorem{remark}[theorem]{Remark}
\newtheorem{corollary}[theorem]{Corollary}
\newtheorem{definition}[theorem]{Definition}

\makeatletter
\newcommand*{\rom}[1]{\expandafter\@slowromancap\romannumeral #1@}
\makeatother

\newcommand{\N}{\mathbb{N}}

\input{macros2}
\newcommand{\p}[1]{\left(#1\right)}
\newcommand{\pp}[1]{\left[#1\right]}
\newcommand{\ppp}[1]{\left\{#1\right\}}

\usepackage{xspace}

\newcommand{\Max}{\max\limits}

\newcommand{\f}[1]{\boldsymbol{#1}}

\newcommand{\EE}{\mathds{E}}

\DeclareMathOperator*{\argmax}{arg\,max}

\usepackage{tikz}
\usetikzlibrary{shapes,arrows}
\usepackage{verbatim}

\tikzstyle{block} = [draw, fill=blue!20, rectangle, 
    minimum height=3em, minimum width=6em]
\tikzstyle{sum} = [draw, fill=blue!20, circle, node distance=1cm]
\tikzstyle{input} = [coordinate]
\tikzstyle{output} = [coordinate]
\tikzstyle{pinstyle} = [pin edge={to-,thin,black}]

\usetikzlibrary{positioning,arrows}
\usetikzlibrary{calc}

\usetikzlibrary{shapes,arrows, arrows.meta,shapes.geometric,positioning,calc}
\definecolor{bl}{HTML}{aecae8}
\definecolor{bl2}{HTML}{7aa7c2}

\begin{document}

\title{Mathematical Framework for Online Social Media Auditing}

\author{Wasim~Huleihel\thanks{W. Huleihel is with the Department of Electrical Engineering-Systems at Tel Aviv university, {T}el {A}viv 6997801, Israel (e-mail:  \texttt{wasimh@tauex.tau.ac.il}).}~~~~~~Yehonathan~Refael\thanks{Y. Refael is with the Department of Electrical Engineering-Systems at Tel Aviv university, {T}el {A}viv 6997801, Israel (e-mail:  \texttt{refaelkalim@mail.tau.ac.il}).}}

\maketitle
\begin{abstract}
Social media platforms (SMPs) leverage algorithmic filtering (AF) as a means of selecting the content that constitutes a user's feed with the aim of maximizing their rewards. Selectively choosing the contents to be shown on the user's feed may yield a certain extent of influence, either minor or major, on the user's decision-making, compared to what it would have been under a natural/fair content selection. As we have witnessed over the past decade, algorithmic filtering can cause detrimental side effects, ranging from biasing individual decisions to shaping those of society as a whole, for example, diverting users' attention from whether to get the COVID-19 vaccine or inducing the public to choose a presidential candidate. The government's constant attempts to regulate the adverse effects of AF are often complicated, due to bureaucracy, legal affairs, and financial considerations. On the other hand SMPs seek to monitor their own algorithmic activities to avoid being fined for exceeding the allowable threshold. In this paper, we mathematically formalize this framework and utilize it to construct a data-driven statistical auditing procedure to regulate AF from deflecting users' beliefs over time, along with sample complexity guarantees. This state-of-the-art algorithm can be used either by authorities acting as external regulators or by SMPs for self-auditing. 
\end{abstract}

\allowdisplaybreaks

\section{Introduction}

Social media platforms (SMPs), e.g., Google, Facebook, and Twitter, are increasingly becoming the prevailing, most easily accessible, and most popular platforms for individual media consumption across the Western world \cite{mitchell2016modern}. Indeed, media platforms act as intermediaries between users and the wealth of information collected from their friends, news, opinion leaders, celebrities, politicians and advertisers. So pervasive and eclectic is the stream of information content collected for each user at any given time that it compels social networks to filter out all but the most relevant information, display it in the user's news feed, and its order of appearance. To that end, in the last decade, social platforms have been adopted various \emph{algorithmic filtering} (AF) methods \cite{Robyn2018AlgorithmicFiltering} to select and sort collections of contents to be shown on their user’s feed.

Notwithstanding the potential of AF to provide users with a richer, more diverse, and more engaging experience, over the past two decades, these methods have been abused by social network platforms to selectively filter user feeds in an effort to maximize their returns (including revenue, user accumulation, popularity gain, etc.). This phenomenon has brought about harmful side effects \cite{devito2017algorithms,bozdag2013bias}. For example, an artificial comment ranking that encourages over-representation of one side's opinion or polarization of opinions \cite{siersdorfer2014analyzing}, has sow hatred between groups \cite{lee2016impact}). Similarly, the prioritization of a specific topic contributes to the dissemination of deliberately disregarded fake news \cite{marwick2017media, chesney2019deep, Collins2019Elliott, Pariser2011}.  Disseminating fake news may sway the presidential election results \cite{Blake2018}. Advertisements that promote products based on erroneous claims regarding the user's interests  \cite{speicher2018potential, sweeney2013discrimination}, leading to some information being more (or less) visible along with many others. Intensive dietary recommendations may cause users to change their own diet \cite{chau2018use,jane2018social}, etc.

The foregoing examples, among many others, embody the potentially damaging fact that subjectively filtering the content to be shown on the user's feed might not overlap with the individual user's or the society's good as a whole, resulting in widespread adverse impacts on both individuals and society \cite{Pariser2011}. This, in turn, heavily impacts users' learning, shapes their thinking and decisions, and ultimately influences how they behave as individuals or as a whole society.  

These negative influences have led to a number of calls for regulatory action by the authorities; however, their increasing enforcement attempts encounter multiple hurdles, such as, legal barriers, cumbersome and entangled bureaucracy, high human resource costs, which usually ends with no concrete results \cite{brannon2019free,klonick2017new,berghel2017lies}.
The legal difficulties are mainly driven by the concern that regulations might limit free speech \cite{klonick2017new,brannon2019free}, infringe on privacy by requiring content disclosure, subjectively define of what is right or wrong media behaviour \cite{obar2015social}, undermine innovation or suppress jobs and revenues (e.g. through advertising restrictions).

Meanwhile, the increasing enforcement of regulations that aims at fining violations encourages the platforms to use self-regulatory methods to prevent unintentional internal activities and avoid penalties \cite{medzini2021enhanced}. Among others, Twitter suspends tens of thousands accounts suspected of being involved in promoting conspiracy theories \cite{Twitter2021self_regulation}. Facebook has set up an independent internal team named  ``Oversight Board" to foster freedom of expression by making principled, independent decisions about contents \cite{OversightBoard2021}; YouTube has removed videos urging violence \cite{YouTube2021self_regulation}.

The suggestion of the notion of an implicit agreement between users and social media platforms is far from being new \cite{manning1989implicit}. This notion draws from a general implicit contract theory \cite{koszegi2014behavioral}, which economists use to explain behaviors that are observed but not justified by competitive market theory. In particular, it has been invoked to explain the reason for users to keep using social media despite data privacy infractions \cite{doi:10.1080/15213269.2019.1598434,sarikakis2017social,agendaworld}. It has also been advanced as a starting point for regulation \cite{quinn2016we}, since it balances the interests of both parties.

As social media become increasingly popular information sources, a fundamental question remains: \emph{Is there a systematic and responsible way to regulate the effect of social media platforms on users learning and decision-making?} Even though it may be possible to do so, due to the many issues raised above, and many other related ones, designing and reinforcing a regulation is still a notoriously difficult open problem \cite{kurbalija2016introduction}. Accordingly, the challenging quest for currently a far reaching fundamental theory for systematic regulatory procedures that satisfy several social, legal, financial, and user related requirements, and its prospective practical ramifications, constitute the main impetus behind this paper. Motivated to guarantee compliance with a consumer-provider agreement, in this paper, we propose a data-driven statistical auditing procedure to regulate AF, which monitors adverse influences on the user learning (and thus on decision-making), while allowing real-time enforcement. 

\subsection{Related Work}
Various attempts aim at regulating content moderation have been proposed over the last few years; however, all of these attempts generally focus on monitoring specific violations of the social platform-user agreement. Specifically, common methods for content moderation fall broadly into one of three categories (see, e.g., \cite{Campbell19,Sina19}): 
\begin{enumerate}
\setlength\itemsep{-0.2em}
    \item \textit{Content control}, which aims at tagging or removing suspicious items. However, the ability of AI algorithms to identify such rough items grows more slowly than the ability to create them \cite{Paschen19}, and objectivity of human content control is often less trusted \cite{Anderson17}. Content control strategies include: drawing a line in the sand (e.g., determining whether discrimination has occurred by thresholding the difference between two proportions \cite{chouldechova2017fair}); detecting hate speech (e.g., using deep learning technique \cite{DBLP:journals/corr/abs-2106-00742,8669073} or NLP clustering methods \cite{davidson2017automated}); or finding the origin of the content (e.g., reducing fake news by whitelisting news sources \cite{berghel2017lies} or detecting the sources that generate misleading posts \cite{racz2020rumor}).
    \item \textit{Transparency}, where users are required to provide lawful identification. This approach imposes a serious toll on user privacy and anonymity, while not even necessarily stopping unintended spread of misinformation. 
    \item \textit{Punishment}, where the network provider or the state impose penalties for malicious spreading of fake information. This extreme approach is clearly the least desirable from both privacy and human-rights perspectives. 
\end{enumerate}

Most related to our work are \cite{cen2020regulating} and \cite{cen2023userdriven}. Specifically, in \cite{cen2020regulating} the concept of ``counterfactual regulations" was proposed and analyzed. Counterfactual regulations deal with regulatory statements of the form: ``The platform should produce similar feeds for given users who are identical except for one single property". The users differentiating property could be, for example, gender, religion, left or right wing affiliation, age, among many others. Accordingly, \cite{cen2020regulating} proposed an auditing procedure to test whether a counterfactual regulation statement is met or not, under a certain i.i.d. observational model. More recently, \cite{cen2023userdriven} introduced the notion of baseline/reference feed as ``the content that a user would see without filtering".\footnote{The idea of a baseline feed was originally proposed in an old version of \cite{cen2020regulating}, which can be found in \cite{cen20OLD}.} Then, they studied the problem of regulating AF with respect to this baseline, and proposed a framework and a procedure for regulating and auditing SMPs with respect to such a baseline. As we explain below in detail, our paper follows some of the general ideas in \cite{cen20OLD} and \cite{cen2020regulating}, but deviates in the way the setting is formulated and analyzed.

Finally, research and modeling of counterfactual regulation draw parallel ideas from the differential privacy literature \cite{dwork2006calibrating,dwork2014algorithmic}, as in the case of comparing outcomes under different interventions  \cite{wasserman2010statistical}. While our paper addresses questions similar to those studied in social learning and opinion dynamics, e.g. \cite{acemoglu2011bayesian,molavi2018theory,banerjee1992simple}, it is distinct from this literature in the sense that our research focuses on the question of how the flow of information, mediated by social networks, leads to undesirable biases in the way users learn and, consequently, to a detrimental change in their decision-making and ultimately in their actions. Furthermore, this is accomplished without the need to actually access the users' beliefs, actions, or thoughts.

\subsection{Main Contributions}

Our main goal is to develop an auditing procedure for content moderation over social networks. We split this subsection into two parts: the first focuses on our conceptual contributions to the general area of social media regulation, while the second discusses our technical contributions.
\subsubsection{Conceptual Contributions}
\label{sec:Conceptual_Contributions}

\paragraph{Unifying framework.} Following the lead of \cite{cen20OLD} and \cite{cen2020regulating}, we formulate a statistical unifying framework for online platform auditing. This framework considers the three involved parties: platform, users, and an auditor, all interacting and evolving over time (see, Figure~\ref{fig:1}). At each time point, the platform shows its users collections of content, known as ``filtered feeds." As each user in the platform browses through his own feed, he implicitly forms a belief, and ultimately modifies his actions. The auditor's meta-objective is to moderate the effect of socially irresponsible externalities caused by the AF's effect on user learning and decision-making, either as individuals or as a society. To that end, the platform supplies the auditor with anonymous data of two types: filtered and reference. The latter is constructed by ignoring any aspect of a platform's fiscal motivation, thus representing a natural/fair filtering of content rather than a subjective form of filtering (see, Section~\ref{section-Framework}, for a precise definition), prioritizing the users' experience. We show that the auditor's task can be formulated as a certain closeness testing problem (see, e.g., \cite{daskalakis2018distribution,canonne2021price}). In addition to the filtered vs. reference approach above, similarly to \cite{cen2020regulating}, we also study counterfactual regulations.

\paragraph{Automatic online auditing procedure.}
We propose an auditing procedure that does not require any prior explicit regulation statement. The auditing procedure monitors any damaging influence on the users' decision-making over a predefined adjustable time-frame, compared to what it would have been without subjective filtering of the users' feeds, namely, under a natural/fair content filtering. This is accomplished by formulating a measure called ``belief-variability", which estimates the influence of the AF on the beliefs of all the users. Using this variability, we then formulate the auditor's objective as a sequential hypothesis testing problem. As a binary hypothesis tester, the auditor examines whether the platform exceeds a tunable threshold of acceptable values of this estimated measurement of influence, doing so over a predefined time frame with a given confidence level. The auditor outputs whether or not regulation is being complied with, meaning whether public opinion is being biased or not. For example, this auditing procedure could easily detect the intensive promotion of a presidential candidate via posts, advertisements, the prioritization of related user comments, artificial adversarial users, or polarized recommendations. Finally, we propose an auditing procedure for deciding whether a platform complies with a given counterfactual regulation statement over the course of time.

We next highlight the main differences and contributions compared to \cite{cen2020regulating} and \cite{cen2023userdriven}. Specifically, both of these papers follow a ``worst-case" approach, where auditing is designed to prevent violations associated with (a hypothetical) ``most gullible" user, i.e., the user whose decisions are most influenced by AF. The idea is that if this user passes regulation, then all other users will pass regulation as well. We instead propose a ``global" approach, where we average the influence of the platform's AF over a set of users. It should be clear that each approach has its own advantages and disadvantages. For example, the worst-case approach might be sensitive to adversarial users; in real-world SMPs, where any party is free to create a user without any supervision, a set of adversarial users can act as more naive/gullible than the most gullible user already in existence, and thus fool the auditor. Also, since the most gullible user is model driven (and not chosen from the data) then he/she might be unrealistically ``too gullible", and then the auditor will announce false alarm violations excessively. Finally, the worst-case approach prevents all users from gradually changing their opinions. This is because, under this approach, the auditing process will immediately result in a violation when the most gullible user alters its opinion slightly. As a result, all other users will not have the opportunity to make slow and natural changes to their opinions, as they would with our average approach. In some sense, the above problematic issues are less severe/relevant in our average approach. It should be emphasized, however, that the outcome of any approach would depend on how seriously the platform engages in conversations on designing the test, model family, and reference feeds. 

Another difference that we would like to emphasize is that the probabilistic setting considered in our paper is different from the one in \cite{cen2020regulating} and \cite{cen2023userdriven}. Specifically, in those papers, an i.i.d. time-independent generative model was assumed for the filtered (and reference) feeds. This implies that feeds are statistically independent, and excludes violations of regulation over time. In ``real-world" cases, this approach may be inherently challenging. Indeed, in cases where regulations must be enforced over time, the procedures in \cite{cen2020regulating} and \cite{cen2023userdriven} must be repeated endlessly. Furthermore, this allows an ``uncooperative" platform to comply with regulation at a specific time when being tested, but not at any other time. In our paper, on the other hand, as an initial attempt and approximation to resolve the above issues, we follow a more complicated time-dependent Markovian model. As so, our auditing procedures, analysis, and results are inherently different from those in the aforementioned papers. 

\subsubsection{Technical Contributions} \label{Contribution}

In addition to formulating a mathematical model for social media auditing, our paper contributes to the study of the \emph{closeness testing problem}. The closeness testing problem have been extensively studied in the past few years (see, e.g., \cite{daskalakis2018distribution,batu2013testing,chan2014optimal,acharya2015optimal}), as well as its extended version, the tolerant closeness testing problem (e.g., \cite{daskalakis2018distribution,canonne2021price}). The vanilla form of the later is as follows. We are given i.i.d. sample access to distributions $P$ and $Q$ over $[n]$, and bounds $\varepsilon_{2}>\varepsilon_{1} \geq 0$, and $\delta>0$. The task is to distinguish with probability of at least $1-\delta$ between $\|P-Q\|_{1} \leq \varepsilon_{1}$ and $\|P-Q\|_{1} \geq \varepsilon_{2}$, whenever $P, Q$ satisfy one of these two inequalities. In our setting, samples (or, feeds) are assumed to be generated from a certain Markovian probabilistic model (rather than being i.i.d.). Testing Markov chains is a new and active area of research with a number of remarkable recent results, such as testing symmetric Markov chains \cite{daskalakis2018testing}, testing Ergodic Markov chains \cite{wolfer2019minimax,wolfer2020minimax} or testing irreducible Markov chains \cite{pmlr-v132-chan21a}. In this paper, we construct a method to solve a generalized form of the two problems above. Specifically, rather than a single pair of distributions, we are given samples from multiple pairs (see, \cite{ReutDana} for a related testing problem) of hidden irreducible Markov chains, and we need to decide whether the total sum of distances between these hidden pairs of chains is $\varepsilon_{1}$-close, or $\varepsilon_{2}$-far away. Similarly to majority of the papers mentioned above, we focus on the case where probabilistic distance measure is $\ell_{\infty}$, with the understanding that other metrics can be analyzed. We propose a testing algorithm to the problem above, along with sample and complexity guarantees. It turns out that a major part of the analysis of our algorithm is related to the study of the covering time of random walks on undirected graphs \cite{pmlr-v132-chan21a}. Specifically, we obtain an upper bound on the time it takes for multiple parallel random walks to cover each state a given number of times. Our analysis might be of independent interest. 

\subsection{Notations}
For a positive integer $m$, we denote $[m] \equiv\{1,2, \ldots, m\}$. The underlying space in the paper is $\mathbb{R}^{n}$, i.e., the space of all real-valued $n$ length column vectors endowed with the dot product $\langle\mathbf{x}, \mathbf{y}\rangle=\mathbf{x}^{T} \mathbf{y}$. For $p \geq 1$, The $\ell_{p}$-norm of a vector $\mathbf{x} \in \mathbb{R}^{n}$ is given by $\|\mathbf{x}\|_{p} \equiv \sqrt[p]{\sum_{i=1}\left|x_{i}\right|^{p}}$. The $\ell_{\infty}$-norm of a vector $\mathbf{x} \in \mathbb{R}^{n}$ is $\|\mathbf{x}\|_{\infty}=\max _{i=1,2, \ldots, n}\left|x_{i}\right|$. The $p$-norm of matrix $\bA$ induced by vector $p$-norms is defined by $\|\bA\|_{p}=\sup _{x \neq 0} \frac{\|\bA x\|_{p}}{\|x\|_{p}}$. In the special cases of $p=1, \infty$, the induced matrix norms can be computed or estimated by $\|\bA\|_{1}=\max _{1 \leq j \leq n} \sum_{i=1}^{m}\left|a_{i j}\right|$,
which is simply the maximum absolute column sum of the matrix; $\|\bA\|_{\infty}=\max _{1 \leq i \leq m}\sum_{j=1}^{n}\left|a_{i j}\right|$, which is simply the maximum absolute row sum of the matrix.
$\mathbf{e}$ is used to denote the vector of all ones and $\mathbf{0}$ is the vector of all zeros. We denote by $\left|S\right|$ the number of element in the set $S$. The function $a:\left(\N\times\N\right)\rightarrow\{\{\N\},\{\N\}\}$ takes a pair of elements and return set containing those two element.
The function $a_f:\left(\N\times\N\right)\rightarrow\{\N\}$ takes a pair of elements and return the first element in the pair (first coordinate). Similarly, the function $a_s:\left(\N\times\N\right)\rightarrow\{\N\}$ takes a pair of elements and return the second element in the pair (first coordinate). Finally, we let $\Delta^{n}$ be the $n$-dimensional probability simplex.

\section{Framework: Setup and Goal}
\label{section-Framework}

In this section, we formalize mathematically our framework, including the setup and goals. Here, we opted to keep the exposition simple and concise, by presenting only the essential ingredients of our model which are needed for our main results. However, we refer the interested reader to the appendix, where we include a detailed and consistent construction of our framework, with deeper discussions and motivations for our definitions and assumptions. 

\subsection{The setup}\label{subsec:setup}

Consider a system with the following three parties: a SMP, a \emph{user}, and an \emph{auditor}, as illustrated in Figure~\ref{fig:1}. At each time step $t\in\mathbb{N}$, the platform shows each user a collection of contents (e.g., posts, videos, photos, ads, etc.) known as \emph{filtered feeds}. We denote the filtered feed shown to user $u\in [\s{U}]$ at time $t\in\mathbb{N}$ by $\bX^{\s{F}}_{u}(t)$, and assume that it consists of $\s{M}\in\mathbb{N}$ pieces of contents, namely, $\bX^{\s{F}}_{u}(t)=\{\f{x}^{\s{F}}_{1,u}(t),\ldots,\f{x}^{\s{F}}_{\s{M},u}(t)\}$, where $\f{x}^{\s{F}}_{j,u}(t)\in\calX$ denotes a piece of content, for $1\leq j\leq\s{M}$. 

Generally speaking, the AF mechanism is not known and should not be disclosed to the auditor. Nonetheless, it should be clear that for the auditor to be able to inspect the SMP, something about the feeds generation process must be assumed. From the auditor's point of view, the platform is a sequential feeds generating system, relying on a probabilistic relationship of the current feed conditioned on the previous feeds. Specifically, in this paper, we assume that the feeds are generated at random according to a quasi-Markov homogeneous model; we divide the time horizon into batches, and assume that in each batch, the platform's AF process is modeled as a large probabilistic state machine. One can think of these batches as time interval where the platform collects new data to create new successive feeds. 

Mathematically, let $\s{T}_{\s{total}}\in\mathbb{N}$ denote the time horizon, which determines how far into the past the auditor scrutinizes the platform's behavior. Assume we have $\s{B}\in\mathbb{N}$ batches each of length $\s{T}\in\mathbb{N}$, such that in batch $b\in[\s{B}]$ we have a time sampling sequence $b\cdot\s{T}<t_{0,b}<t_{1,b}<\dots<t_{\s{T},b}\leq (b+1)\cdot\s{T}$. In each batch, from the auditor's point of view, the piece of content $\f{x}_{\ell,u}^{\s{F}}(t_{i,b})$, at time $t_{i,b}$, for $\ell\in[\s{M}]$, is drawn from a first-order irreducible Markov chain, namely, $ \pr(\f{x}^{\s{F}}_{\ell,u}(t_{i,b})\vert\f{x}^{\s{F}}_{\ell,u}(t_{0,b}),\ldots,\f{x}^{\s{F}}_{\ell,u}(t_{i-1,b})) = \pr(\f{x}^{\s{F}}_{\ell,u}(t_{i,b})\vert\f{x}^{\s{F}}_{\ell,u}(t_{i-1,b}))$, and $\pr(\f{x}^{\s{F}}_{\ell,u}(t_{i,b})=s_2\vert\f{x}^{\s{F}}_{\ell,u}(t_{i-1,b})=s_1)\triangleq Q_{u,b}(s_1,s_2)$, for any two possible states $s_1,s_2\in\calX$. We denote the transition probability matrix by in batch $b\in[\s{B}]$ by $\s{Q}^{\f{F}}_{u,b}=\left[Q_{u,b}(s_1,s_2)\right]_{s_1,s_2 \in\calX}$. We assume further that the $\s{M}$ Markov trajectories are i.i.d. Note that over different intervals, indexed by $b$, the filtering process could be transformed into a new state machine subjected to a different transition probabilities. For example, this transformation may occur over time when new external data incur noticeable changes in the platform's reward. Thus, in the $b$th batch, the observed feeds are,
$$\underset{\text{Feed 1}}{\left\{\f{x}^{\s{F}}_{l,u}(t_{0,b})\right\}_{l=1}^{\s{M}}}, \underset{\text{Feed 2}}{\left\{\f{x}^{\s{F}}_{l,u}(t_{1,b})\right\}_{l=1}^{\s{M}}},
\ldots,
\underset{\text{Feed } \s{T}}{\left\{\f{x}^{\s{F}}_{l,u}(t_{\s{T},b})\right\}_{l=1}^{\s{M}}}.$$ 
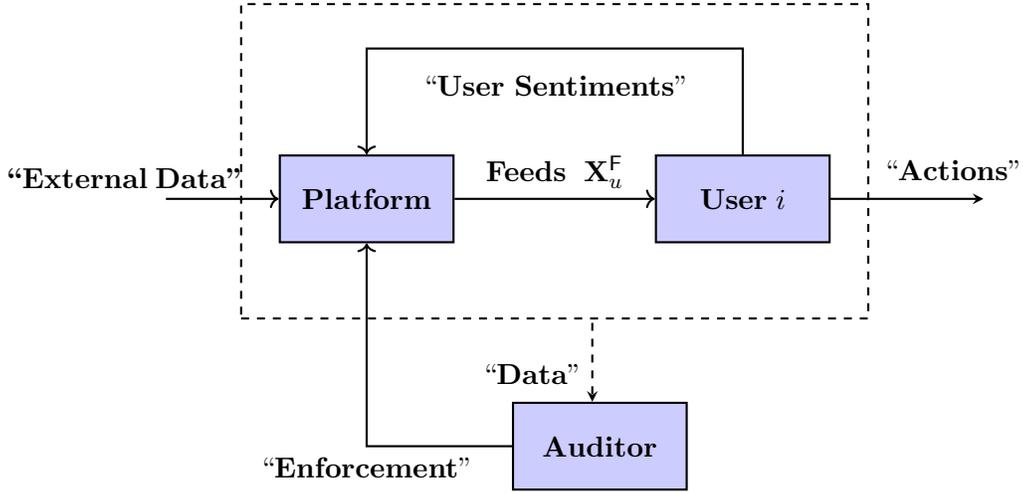
\begin{figure*}[t]
\centering\scalebox{1}{
\begin{tikzpicture}[thick,black!70!black]
\path
(0,0)       node[block] (N) {$\textbf{Platform}$}
++(0:5) node[block] (C) {$\textbf{User}\;i$}
++(-120:3.8)    node[block] (O) {$\textbf{Auditor}$};

\draw[<-] (N.west)--+(180:1.5) node[above,black,xshift=-0.5cm] {$\textbf{``External}\;\textbf{Data"}\;$};
\draw[->] (N)--(C) node[midway,above,black] {$\textbf{Feeds}\;$ $\bX_u^{\s{F}}$};
\draw[<-] (N)--++(90:2) coordinate (A)-|(C);
\path (A)--(C) node[midway,above,align=left,black,xshift=0.6cm] {``$\textbf{User}$ $\textbf{Sentiments}$"};
\draw [->] (O) -| node[midway,below,align=left,black] {``$\textbf{Enforcement}$"} (N);
\draw[thick,dashed]     ($(N.north west)+(-0.5,2)$) rectangle ($(C.south east)+(0.5,-1)$);
\draw [-stealth,thick,dashed](3,-1.65) -- (3,-2.7) node[above,black,xshift=-0.8cm,yshift=0.1cm] {``$\textbf{Data}$"};
\draw [-stealth,thick](6.15,0) -- (8.2,0) node[above,black,xshift=-0.4cm,yshift=0.1cm] {``$\textbf{Actions}$"};
\end{tikzpicture}}  
\caption{An illustration of the interaction between the platform, the user, and the auditor.}
\label{fig:1}
\end{figure*}

\paragraph{Reference feeds.} 
Following \cite{wachter2019right,ghosh2019new,cen20OLD,petty2000marketing}, we define a reference boundary that is formed based on the users consent, and its location is determined by domain experts. While user $u$'s filtered feed $\bX_{u}^{\s{F}}(t)$ at time $t$ is chosen by the platform in a certain reward-maximizing methodology, the \emph{reference feeds} $\bX_u^{\s{R}}(t)$ could have been hypothetically selected by the platform if it strictly followed the consumer-provider agreement. These reference feeds are specific to each user $u\in[\s{U}]$ and time $t$. In this scenario, the platform would construct the feed based solely on the user's interests, without any subjective preferences influencing the content selection. Essentially, the only natural situation where the platform can filter content without introducing any subjective bias into the user's decision-making process and actions is by selecting feasible content that maximizes the user's benefit/reward. This approach ensures that the user's feed reflects their own preferences, which may align with the platform's benefits at times, but not necessarily always. Mathematically, the user's exclusive benefit is quantified by a personal reward function that encompasses only the components measuring the user's benefits. Similarly to the filtered feeds, we assume a Markovian generative model for the reference feeds, and denote by $\s{P}^{\f{R}}_{u,b}\triangleq\left[P_{u,b}(s_1,s_2)\right]_{i,j \in\calX}$ the corresponding matrix transition probabilities in batch $b\in[\s{B}]$. In the appendix, we give examples of how the filtered and reference feeds are constructed by means of a certain reward function maximization, and elucidate the difference between the filtered and reference feeds. It should be emphasized that the specific reference feed construction hinted above, and described in more detail in the appendix, is just one possible example; our results and algorithms only require that there is some fixed reference feed (per user).

\paragraph{Counterfactual regulation.}

In addition to the ``filtered vs. reference" approach, we will also analyze the following alternative auditing framework. Let $\calS$ be a \emph{regulatory statement} that an inspector (or, perhaps, the platform itself) wish to test. For example, $\calS$ could be: ``\emph{The platform should produce similar feeds, in the course of a given time horizon $\s{T}$, for users who are identical except for property $\mathscr{P}$}", where $\mathscr{P}$ could be ethnicity, sexual orientation, gender, a combination of these factors, etc. Let $\calU_{\mathscr{P}}\subset[\s{U}]\times[\s{U}]$ be a subset of pairs of users that comply with $\mathscr{P}$. Then, for any pair of users $(i,j)\in\calU_{\mathscr{P}}$, the inspector's objective is to determine whether the platform's filtering algorithm cause user $i$'s and user $j$'s beliefs and actions to be significantly different. We formulate this objective rigorously in the next section. We mention here that a similar approach to the above was proposed recently in \cite{cen2020regulating}, assuming a time-independent i.i.d. model. 

\subsection{Auditor's goal}\label{subsec:regulatorgoal}

\paragraph{Average violation.} We now define the meaning of ``violation" from the auditor's perspective. Let $\calU\subseteq[\s{U}]$ be a certain subset of users. We define the \emph{total filtering-variability metric} as,
\begin{align}
    \mathds{V}_{\s{filter}} &= \frac{1}{|\calU|}\sum_{u\in\calU}\max_{i\in\calX}\s{d}_{\s{TV}}\p{P_{u,b}(i,\cdot),Q_{u,b}(i,\cdot)} \\
    &= \frac{1}{|\calU|}\sum_{u\in\calU}\max _{i\in\calX} \left\|\mathbf{P}^{\f{R}}_{u,b}(i)- \mathbf{Q}^{\f{F}}_{u,b}(i) \right\|_{1}\\
    &= \frac{1}{|\calU|}\sum_{u\in\calU}\left\|\s{P}^{\f{R}}_{u,b}-\s{Q}^{\f{F}}_{u,b}\right\|_{\infty},\label{eqn:VfilterObjective}
\end{align}
where $\mathbf{P}^{\f{R}}_{u,b}(i)\triangleq[P_{u,b}(i,j)]_{j\in\calX}$ and $\mathbf{Q}^{\f{F}}_{u,b}(i)\triangleq[Q_{u,b}(i,j)]_{j\in\calX}$. We discuss the choice of the above metric in the appendix. Without loss of generality, in the rest of this paper, we focus on the special case where $\calU = [\s{U}]$. Also, we will consider a single specific interval for testing, say, $\{t_0,t_1,\ldots,t_{\s{T}}\} = [\s{T}]$, and therefore drop the dependency of the above notations on the batch index $b$. The underlying assumption here is that $\s{T}$ is sufficiently large so as to allow for reliable testing, as dictated by our sample complexity guarantees, presented in the next section. An interesting question is to consider the case where the batch sizes are unknown, and then more sophisticated sequential/adaptive testing algorithms are needed.

\paragraph{Testing.} Following the above, from the auditor's perspective, we define a violation event as the case where 
 $\mathds{V}_{\s{filter}}$ is ``unusually large". Specifically, the audit's decision task is formulated as the following hypothesis testing problem,
  \begin{align}
   \calH_{0}:\mathds{V}_{\s{filter}}\leq\varepsilon_1\quad\s{vs.}\quad\calH_{1}:\mathds{V}_{\s{filter}}\geq\varepsilon_2,\label{eqn:HTbelief3}
  \end{align}
where $\varepsilon_2>\varepsilon_1\geq0$ govern the auditing strictness. Devising successful statistical tests which solve \eqref{eqn:HTbelief3} with high probability, guarantee that whenever the auditor decision is $\calH_0$, then the platform honors the consumer-provider agreement, since the beliefs and actions are indistinguishable under the filtered and reference feeds. Conversely, rejecting $\calH_0$ with high confidence implies that AF causes significantly different learning outcomes. Calculating $\mathds{V}_{\s{filter}}$ requires knowledge of the filtering and reference distributions; a condition rarely met in practice. Accordingly, the auditor needs to solve \eqref{eqn:HTbelief3} using only samples from these distributions; we assume that for $t\geq1$ the auditor observes the filtered and reference feeds $\{\bX^{\s{F}}_{i}(t),\bX^{\s{R}}_{i}(t)\}$, for all users $u\in[\s{U}]$, and utilize these to test for violations. In practice, it might be challenging for the auditor to have both the reference and filtered feeds at hand. As so, it is an interesting question for future research to analyze the scenario where this full information is not available, e.g., only partial and perhaps quantized/noisy observations are given. Note that this type of hypothesis testing problem is reminiscent of the well-studied \emph{tolerant closeness testing} problem (see, e.g., \cite{daskalakis2018distribution,canonne2021price}). We are now in a position to state the testing problem faced by the auditor.
\begin{problem}[Auditor testing]\label{prob:5}
    Fix $\varepsilon_1,\varepsilon_2\in(0,1)$ and $\delta\in(0,1)$ with $\varepsilon_1<\varepsilon_2$. Given a set of $t_{\s{T}}$ pairs of Markovian trajectories $\left[\left(\bX^{\s{F}}_{u}(t_1), \bX^{\s{R}}_{u}(t_1)\right),\dots,\left(\bX^{\s{F}}_{u}(t_{\s{T}}), \bX^{\s{R}}_{u}(t_{\s{T}})\right)\right]$ drawn from an \emph{unknown} corresponding pair of Markov chains $\p{\s{Q}^{\f{F}}_{u},\s{P}^{\f{R}}_{u}}$, for each user $u\in\s{U}$, an $(\varepsilon_1, \varepsilon_2, \delta)$-sum of pairs tolerant closeness testing algorithm outputs \texttt{YES} if $\mathds{V}_{\s{filter}}\leq \varepsilon_1$ and `\texttt{NO} if $\mathds{V}_{\s{filter}}\geq\varepsilon_2$, with probability at least $1-\delta$.
\end{problem}
As we mentioned earlier, the testing problem above is similar to the well-studied Markov tolerant closeness testing problem (e.g., \cite{pmlr-v132-chan21a}). Nonetheless, the vanilla setting of this type of testing, is simpler than the one we are after, mainly because in our problem we deal with a \emph{sum} of the distances between pairs of latent Markov chains, rather than a single distance, as it is in the standard setting. Finally, Figure~\ref{fig:smp_process} illustrates the filtered vs. reference testing scheme considered in this paper.

\begin{remark}[Worst-case violation]
As we have mentioned in the introduction and above, in this paper we focus on a global approach by averaging the influence of the platform on the users. Here, we would like to mention that using the same techniques we develop in this paper, a worst-case approach, in the same vein as in \cite{cen2020regulating}, can be analyzed as well. Specifically, the idea in the worst-case approach is that if the platform's influence on the most gullible user's decision-making exceeds a predefined threshold, then it would mean a violation of the platform-user agreement. Now, within the proposed framework, we define this most gullible user as,
$$
u_{\s{gullible}}=\argmax_{u\in[\s{U}]}\left\|\s{P}^{\f{R}}_{u,b}-\s{Q}^{\f{F}}_{u,b}\right\|_{\infty},
$$
i.e., the platform's influence on his feed is the most significant. Intuitively, if the filtered feed satisfies regulation for the most gullible user, then it satisfies regulation for all other users whose learning is, by definition, less affected by the filtered feed, in the above sense. Accordingly, the auditor testing problem can be formulated as testing between
\begin{align}
    \calH^{\s{worst}}_{0}:\max_{u\in[\s{U}]}\left\|\s{P}^{\f{R}}_{u,b}-\s{Q}^{\f{F}}_{u,b}\right\|_{\infty}\leq\varepsilon_1\hspace{0.2cm}\quad\s{vs.}\quad\calH^{\s{worst}}_{1}:\max_{u\in[\s{U}]}\left\|\s{P}^{\f{R}}_{u,b}-\s{Q}^{\f{F}}_{u,b}\right\|_{\infty}\geq\varepsilon_2.
\end{align}    
\end{remark} 

\begin{figure}[t]
\centering\scalebox{1}{
\begin{tikzpicture}
	\tikzstyle{block} = [outer sep=0pt, fill=blue!20, fill=bl,draw,rectangle, minimum width=2cm, minimum height=1.4cm]
	\node[block, fill=blue!20] (SMP) {$\s{SMP}$};
	\node[block, fill=blue!20, below=1cm of SMP, align=center] (UF) {$\s{Reference}$\\ $\s{Filter}$};
	\node[fill=bl,draw, fill=blue!20, minimum width=2cm, minimum height=3.5cm, xshift=4cm, align=center] (SN) at ($(SMP)!.5!(UF)$) {$\s{Filtered}$\\ $\s{Reference}$\\ $\s{Auditor}$};
	\coordinate (p) at (SN.west);
	\coordinate (p1) at (SMP.east);
	\coordinate (p2) at (UF.east);
	\draw[-stealth', line width=1.1pt] (p1)-- node[above] (zf) {$\mathbf{X}^{\s{F}}$} (p1-|p);
	\draw[-stealth',line width=1.1pt] (p2)-- node[above] {$\mathbf{X}^{\s{R}}$} (p2-|p);
	
	\draw[-stealth', line width=1.1pt] (SN)-- node[above, pos=.37] {\texttt{YES/NO}} +(4,0);
	\draw[stealth'-,line width=1.1pt] (SMP)-- node[above,pos=.68] {$\s{External}$ $\s{Data}$ $\mathbf{X}$} node[pos=.2, circle, scale=.6pt,fill=black] (n1) {} +(-6,0);
	\draw[-stealth', line width=1.1pt] (n1.center)|- node[pos=.51, circle, scale=.6pt,fill=black] (n2) {} (UF.west);
	\draw[bl2,dashed, color = black, line width=1.5pt] ([xshift=-4mm, yshift=1.5cm]n1) rectangle ([xshift=2.2cm,yshift=-.4cm]SN.south east);
	
	\node[above=6.5mm, xshift=6.5mm,yshift=-.15cm] at (zf) {$\s{Auditing}$ $\s{Procedure}$};
\end{tikzpicture}}
\caption{An illustration of the auditing procedure. The SMP and the uniformal filter get as an input the external data procedure, then output the filtered and the reference feeds, respectively. Both feeds are seen by the auditor, where the last outputs ``\texttt{YES}" when the regulation is not violated, or ``\texttt{NO}" otherwise. }
\label{fig:smp_process}
\end{figure}
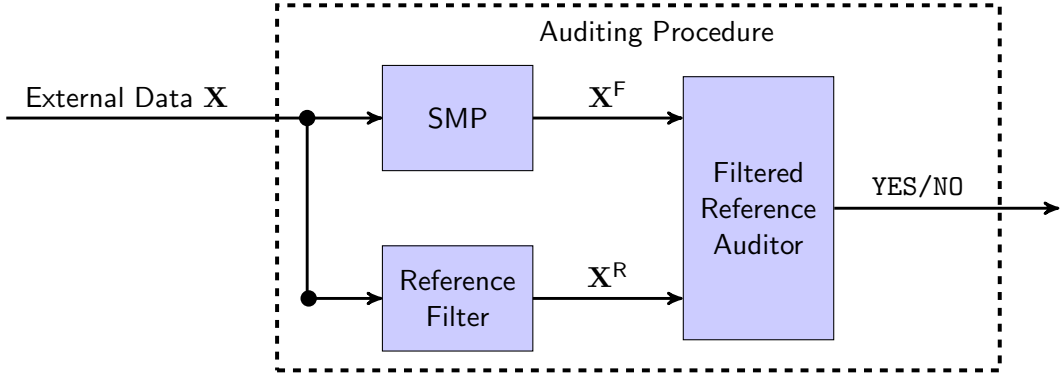

\subsection{Auxiliary definitions and lemmas}
\label{sec:Auxiliary Definitions and Lemmas}

This subsection is devoted to present several notations, definitions, and a lemma that will be needed to present our main results. As mentioned in the previous subsection, the problem of closeness testing of a single pair of Markov chains was considered in, for example, \cite{pmlr-v132-chan21a}; it was shown that the testing algorithm and sample complexity depend on the $k$-cover time. The former is defined as the first time that a random walk has visited every state of the Markov chain at least $k$ times, while the later is the maximization of the expectation of this random variable over all initial states. As a natural generalization, we define the $\ell$-joint-$k$-cover time, as the expected time it takes for $\ell\geq1$ independent random walks to cover all states at least $k$ times. In the language of our framework, an interpretation of this $\ell$-joint-$k$-cover time is the expected time it takes the platform to show all users all feasible contents.

\begin{definition}[$\ell$-joint-$k$-cover time]Let $\s{Z}_{1,1}^{\infty},\s{Z}_{2,1}^{\infty},...,\s{Z}_{\ell,1}^{\infty}$ be $\ell$-independent infinite trajectories drawn by the same Markov chain $\mathscr{M}$. For $t\geq1$, let $\{\calN_i^{\s{Z}_{j}}(t), \forall i \in[n]\}$ be the counting distribution of states $i\in[n]$ appearing in the subtrajectory $\s{Z}_{j,1}^{t}$ up to time $t$. For any $k,\ell \in \mathbb{N}$, the random $\ell$-joint-$k$-cover time $\tau_{\s{cov}}^{(k)}(\ell;\mathscr{M})$, is the first time when all $\ell$ independent random walks have jointly visited every state of $\mathscr{M}$ at least $k$ times, i.e., 
\begin{align}
\tau_{\s{cov}}^{(k)}(\ell;\mathscr{M}) \triangleq \inf \left\{t\geq0: \forall i \in[n],\sum_{j=1}^{\ell}\calN_{i}^{\s{Z}_{j}}(t)\geq k\right\}.    
\end{align}
Accordingly, the $\ell$-joint-$k$-cover time is given by
\begin{align}
t_{\s{cov}}^{(k)}(\ell;\mathscr{M})  \triangleq \max _{\mathbf{v} \in[n]^\ell} \mathbb{E}\left[\tau_{\s{cov}}^{(k)}(\ell;\mathscr{M}) \mid \s{Z}_{1,1}=v_1,\s{Z}_{2,1}=v_2,...,\s{Z}_{\ell,1}=v_\ell\right],    
\end{align}
where the coordinates of $\mathbf{v}=(v_1,v_2,\dots,v_\ell)\in[n]^\ell$ correspond to initial states.\end{definition} 
Throughout the paper, we will also use the notation $t_{\s{cov}}^{(k)}(\ell;\s{P})$ to refer to $t_{\s{cov}}^{(k)}(\ell;\mathscr{M})$, where $\s{P}$ denotes the matrix transition probabilities of $\mathscr{M}$. For simplicity of notation, we denote $t_{\s{cov}}(\mathscr{M}) \equiv t_{\s{cov}}^{(1)}(1;\mathscr{M})$. In addition, unless we explicitly deal with two different chains, we omit the dependency of $t_{\s{cov}}^{(k)}(\ell;\mathscr{M})$ on $\mathscr{M}$ and use $t_{\s{cov}}^{(k)}(\ell)$ instead. We denote by $\bf{\pi}$ the stationary distribution of $\mathscr{M}$, and accordingly we define the \emph{mixing time} as $t_{\s{\s{mix}}}(\mathscr{M}) \triangleq \min \left\{t \geq 1: \max _{\mu \in \Delta_{d}}\s{d}_{\s{TV}}(\mu\mathscr{M}^{t},\pi) \leq 1/4 \right\}$. Finally, we denote the minimum stationary probability as $\pi_\star = \min_{i\in[n]}\pi_i$. Our goal is to bound $t_{\s{cov}}^{(k)}(\ell;\mathscr{M})$ in terms of $t_{\s{cov}}(\mathscr{M})$. To date, studies have focused on one of the following two separate cases:
\begin{enumerate}
    \item Upper bounding the expected time required to cover all $n$ states of an irreducible Markov chain, $k$ times, with a single random walk, given by $t_{\s{cov}}^{(k)}(1;\mathscr{M})$, in the terms of $t_{\s{cov}}(\mathscr{M})$. Specifically, it is shown in \cite{pmlr-v132-chan21a} that for irreducible Markov chain,
    \begin{align}
    t_{\s{cov}}^{(k)}(1;\mathscr{M})=\calO\left(t_{\s{cov}}(\mathscr{M})\log n+\frac{k\log n}{\pi_{\star}}\right).
    \end{align}
    \item Upper bounding the expected time required to cover all $n$ states of some general irreducible Markov chain, with $\ell$ multiple independent random walks, given by $t_{\s{cov}}^{(1)}(\ell;\mathscr{M})$, in the terms of $t_{\s{cov}}$. Many bounds, relying on different assumptions, exist in the literature. For example, combining Theorem 3.2, Lemma 4.3, and  Theorem 4.7 in \cite{rivera_sauerwald_sylvester_2023}, we get that for irreducible Markov chain,
\begin{align}
   t_{\s{cov}}^{(1)}(\ell;\mathscr{M})= \calO\p{t_{\s{mix}}(\mathscr{M})\vee\frac{t_{\s{cov}}(\mathscr{M})\log n}{\ell}}.
\end{align}
\end{enumerate}
For our case, we obtain the following bound on the $t_{\s{cov}}^{(k)}(\ell;\mathscr{M})$, for any $k,\ell\geq1$.
\begin{lemma}\label{$m$-joint-$k$-cover time upper bound}
For any $k,\ell\geq 1$ and irreducible Markov chains $\mathscr{M}$,
\begin{align}
t_{\s{cov}}^{(k)}(\ell;\mathscr{M})= \mathcal{O}\p{\pp{t_{\s{mix}}(\mathscr{M})\vee\frac{t_{\s{cov}}(\mathscr{M})\log n}{\ell}}\log n+\frac{k\log n}{\pi_{\star}}}.   
\end{align}
\end{lemma}

Finally, following \cite{daskalakis2018testing}, for a length $q$ trajectory $\s{Z}_{1}^q$ of an irreducible Markov chain $\mathscr{M}$, and for any state $i\in[n]$, we define the mapping $\psi_{k}^{(i)}(\s{Z}_{1}^q)$ as follows: we
look at the first $k$ visits to state $i$ (i.e., at times $t= t_1,\ldots,t_{k}$ with $\s{Z}_t=i$) and write down the corresponding transitions in $\s{Z}_{1}^q$, i.e., $\s{Z}_{t+1}$. In other words, the mapping returns the $k$ succeeding states of state $i$. We note that every state is visited almost surely, since $\mathscr{M}$ is an irreducible finite-state Markov chain. Therefore, the above mapping defines a proper probability distribution. Most importantly, as we will show later on, this distribution is independent across all different states and/or independent for a particular state $i$ because of the Markov property. 

\section{Main Results}
\label{Regulatory Procedure section}

In this section, we present our main results. Specifically, in Subsection~\ref{sec:Tolerant-Closeness-Tester}, we start by presenting an algorithm, along with sample complexity guarantees, for closeness testing the sum of distances of pairs of discrete distributions using i.i.d. samples. This in turn will serve as a sub-routine in the auditing procedure we propose and analyze in Subsection~\ref{subsec:regulatrotyproceduremain}. Finally, in Subsection~\ref{subsec:Counterfactual} we analyze the counterfactual regulation approach.

\subsection{Warm up: Testing a family of discrete distributions}
\label{sec:Tolerant-Closeness-Tester}

As a warm-up, we start by generalizing the vanilla i.i.d. tolerant closeness testing problem (e.g., \cite{canonne2021price}), to the case where one is given a \emph{set} of pairs of measurements drawn from a \emph{set} of pairs of probability distributions, and is tasked with deciding whether the total sum of distances between these pairs of distributions is close or far away. This problem is formulated mathematically as follows.
\begin{problem}[Sum closeness testing]\label{prob:3}
    Given sample access the pairs of distributions $(P_{u},Q_u)$ over $[n]$, for $u\in[\s{U}]$, and bounds $\varepsilon_{2}>\varepsilon_{1} \geq 0$, and $\delta>0$, distinguish with probability of at least $1-\delta$ between $\sum_{u=1}^{\s{U}}\|P_{u}-Q_{u}\|_{1} \leq |\s{U}|\cdot\varepsilon_{1}$ and $\sum_{u=1}^{\s{U}}\|P_{u}-Q_{u}\|_{1} \geq |\s{U}|\cdot\varepsilon_{2}$, whenever the distributions satisfy one of these two inequalities.
\end{problem}
The vanilla i.i.d. tolerant closeness testing corresponds to  $\s{U}=1$. As we will see in the following subsection, an algorithm to Problem~\ref{prob:3} will serve as a building block to the actual testing problem we are after in Problem~\ref{prob:5}. We next propose a procedure solving the above testing problem along with sample complexity guarantees. We establish first a few notations. Let $\calS_{u,P}$ and $\calS_{u,Q}$ be two sets of $m\in\N$ samples drawn from $P_u$ and $Q_u$, respectively, for all $u\in[\s{U}]$, and let $\calS_P \triangleq \{\calS_{1,P},\ldots,\calS_{\s{U},P}\}$ and $\calS_Q \triangleq \{\calS_{1,Q},\ldots,\calS_{\s{U},Q}\}$. For every $u\in\s{U}$, let $V_{u,i}$ and $\tilde{V}_{u,i}$ count the number of occurrences of symbol $i\in[n]$, in the first and the second sets of $m$ samples (each), sampled from $P_{u}$, respectively. Similarly, we denote $Y_{u,i}$ and $\tilde{Y}_{u,i}$ the corresponding samples from $Q_u$, for every $u\in\s{U}$. As is customary in the literature of distributional testing (e.g., \cite{canonne2021price}), we use the ``Poissonization" trick, and assume that the sample sizes of $P_{u}$ and $Q_{u}$, for every symbol $i\in[n]$, are Poisson-distributed with mean $m$, namely, $V_{u,i},\tilde{V}_{u,i}\sim\s{Poisson}\left(m\cdot P_{u,i}\right)$ and $Y_{u,i},\tilde{Y}_{u,i}\sim\s{Poisson}\left(m\cdot Q_{u,i}\right)$, where $P_{u,i}$ ($Q_{u,i}$) is the probability of symbol $i$ under $P_u$ ($Q_u$). Define, 
\begin{align}
    f_{u,i} \triangleq
     \begin{cases}
       \Max\left\{\sqrt{mn}|P_{u,i} - Q_{u,i}|,n(P_{u,i} + Q_{u,i}),1\right\}, & \quad\text{if}\quad m>n,\\
       \Max\left\{m(P_{u,i} - Q_{u,i}),1\right\}, &\quad\text{otherwise,} \\ 
     \end{cases}\label{fuidef}
\end{align}
where $\tilde{V}_{u,i},\tilde{Y}_{u,i}$ are used to estimate $f_{u,i}$ with $\widehat{{f}_{u,i}}$, defined as,
\begin{align}
\hat{f}_{u,i} \triangleq 
     \begin{cases}
       \Max\left\{\frac{|\tilde{V}_{u,i} - \tilde{Y}_{u,i}|}{\sqrt{m/n}}|,\frac{\tilde{V}_{u,i} + \tilde{Y}_{u,i}}{m/n},1\right\}, & \quad\text{if}\quad m>n,\\
       \Max\left\{\tilde{V}_{u,i} + \tilde{Y}_{u,i},1\right\}, &\quad\text{otherwise.} \\ 
     \end{cases}
\end{align}
Additionally, define $G_{u,i}\triangleq(V_{u,i}-Y_{u,i})^2-V_{u,i}-Y_{u,i}$, and finally, 
\begin{align}
G \triangleq \sum_{u=1}^{\s{U}}\sum_{i=1}^{\s{R}}\frac{G_{u,i}}{\hat{f}_{u,i}}.\label{eqn:G_est}
\end{align}
Consider the routine $\textsc{IIDTester}(\calS_P,\calS_Q,\delta,\varepsilon_1,\varepsilon_2,m,n)$ in Algorithm~\ref{algo:tester_HT_4}. The constant $c>0$ is an absolute constant determined in the course of the analysis. We have the following result.
\begin{theorem}[Sample complexity]
There exists an absolute constant $c>0$ such that, for any $0\leq \varepsilon_{2} \leq 1$ and $0 \leq \varepsilon_{1} \leq c \varepsilon_{2}$, given
\begin{align}
    m = \mathcal{O}\left(\sqrt{\frac{n}{\varepsilon_2^4\delta\s{U}}}+n\frac{\varepsilon_1^2}{\varepsilon_2^4}+n\frac{\varepsilon_1}{\varepsilon_2^2}+\frac{n^{2/3}}{\s{U}\varepsilon_2^{4/3}}\right),
\end{align}
samples from each of $\{P_{u}\}_{u=1}^{\s{U}}$ and $\{Q_{u}\}_{u=1}^{\s{U}}$, Algorithm \ref{algo:tester_HT_4} distinguish between $\sum_{u=1}^{\s{U}}\|P_{u}-Q_{u}\|_{1} \leq \s{U}\cdot\varepsilon_{1}$ and $\sum_{u=1}^{\s{U}}\|P_{u}-Q_{u}\|_{1} \geq \s{U}\cdot\varepsilon_{2}$, with probability at least $1-\delta$.\label{theorem:iid_sample_complexity}
\end{theorem}
\begin{algorithm}[t]
\KwIn{$\s{U}$, $n$, $m$, $\varepsilon_1,\delta$, and samples $\calS_P$ and $\calS_Q$ from $\{(P_{u},Q_{u})\}_{u\in[\s{U}]}$.}
{\bf{Set}} 
$\tau \longleftarrow c \min \left(\frac{m^{3 / 2} \varepsilon_{2}}{n^{\frac{1}{2}}}, \frac{\s{U}m^{2} \varepsilon_{2}^{2}}{n}\right)$\\
{\bf{Compute}} $G$ in \eqref{eqn:G_est}.\\
{\bf{If }} $G<\tau,$ then {\bf{Return}} \texttt{YES}\\
{\bf{Else}} $G\geq\tau,$ then {\bf{Return}} \texttt{NO}\\
\caption{Tolerant closeness tester for the i.i.d. pairs}
\label{algo:tester_HT_4}
\end{algorithm}

\subsection{Filtered vs. reference auditing}\label{subsec:regulatrotyproceduremain}
In this subsection, we present our auditing procedure for the filtered vs. reference feeds approach. 
We denote by $m\left(n,\varepsilon_{1}, \varepsilon_{2},\delta\right)$ the sample complexity of the i.i.d. tester in Algorithm~\ref{algo:tester_HT_4}, and assume that it satisfies the condition in Theorem~\ref{theorem:iid_sample_complexity}. Let $\bar{m}\triangleq m\left(n,\varepsilon_{1}, \varepsilon_{2},\delta/4n\right)$. Consider the auditing procedure in Algorithm~\ref{algo:TolerantClosenessTestingRegu}.
\begin{algorithm}[t]
\KwIn{$\s{T}$, $n\triangleq|\calX|$, $\varepsilon_1,\varepsilon_1,\delta$, $\bar{m}$, and feeds $\{\bX^{\s{R}}_{u}(t),\bX^{\s{F}}_{u}(t)\}_{t=1}^{\s{T}}$, for $u\in[\s{U}]$.}
\KwOut{ \texttt{YES} if $\mathbb{V}_{\s{filter}} \leq \varepsilon_1$ / \texttt{NO} if $\mathbb{V}_{\s{filter}} \geq \varepsilon_2$.}
{\bf{For}} $i \leftarrow 1,2 \ldots\ldots,n $ \\
$\hspace{0.75cm}$Set $\calS^\s{R} \leftarrow \emptyset$ and $\calS^\s{F} \leftarrow \emptyset$\\
$\hspace{0.75cm}${\bf{For}} every user  $u \leftarrow 1,2 \ldots\ldots,\s{U} $ \\
 $\hspace{1.5cm}$ {\bf{If }}  $\sum_{j=1}^{\s{M}}\calN_{i}^{{\f{x}}_{j,u}^{\s{R}}}<\bar{m}$ {\bf{or}}  $\sum_{j=1}^{\s{M}}\calN_{i}^{{\f{x}}_{j,u}^{\s{F}}}<\bar{m}$ \\
$\hspace{2.25cm}$  {\bf{Return}} \texttt{NO}\\
 $\hspace{1.5cm}$ Calculate $\calS^\s{R}_u\leftarrow\cup_{j=1}^{\s{M}}\psi_{\bar{m}}^{(i)}\left(\{{\f{x}}_{j,u}^{\s{R}}(t)\}_{t=1}^{\s{T}}\right)$ and  $\calS^\s{F}_u\leftarrow\cup_{j=1}^{\s{M}}\psi_{\bar{m}}^{(i)}\left(\{{\f{x}}_{j,u}^{\s{F}}(t)\}_{t=1}^{\s{T}}\right)$ \\
 $\hspace{1.5cm}$ Do $\calS^\s{R} \leftarrow \calS^\s{R}\cup \calS^\s{R}_u$ and $\calS^\s{F} \leftarrow \calS^\s{F}\cup \calS^\s{F}_u$ \\
 $\hspace{0.75cm}$ {\bf{If }} $\textsc{IIDTester}(\calS^\s{R},\calS^\s{F},\delta,\varepsilon_1,\varepsilon_2,\bar{m},n)=$ \texttt{NO}\\
     $\hspace{2cm}$  {\bf{Return}} \texttt{NO}\\
{\bf{Return}} \texttt{YES}
\caption{Filtered vs. reference auditing procedure}
\label{algo:TolerantClosenessTestingRegu}
\end{algorithm}
We have the following result.

\begin{theorem}[Sample complexity] Given an $(\varepsilon_1,\varepsilon_2, \delta)$ i.i.d. tolerant-closeness-tester for $n$ state distributions with the sample complexity of $m\left(n,\varepsilon_{1}, \varepsilon_{2},\delta\right)$, then we can $(\varepsilon_1,\varepsilon_2, \delta)$ testing hypothesis \eqref{eqn:HTbelief3} using,
\begin{align}
\s{T}=\calO\p{\max_{u\in[\s{U}]}\max_{\s{W}\in\{\s{Q}_u^{\s{F}},\s{P}_u^{\s{R}}\}}t_{\s{cov}}^{\bar{m}}\left(\s{M};\s{W}\right)\log \frac{\s{U}}{\delta}},\label{eqn:sampleGurantee}
\end{align}
samples per user. 
\label{lemma:m-Jointly-k-cover time and tolerant-closeness-testing Markov chains}\end{theorem}
Note that in step 8 of Algorithm~\ref{algo:TolerantClosenessTestingRegu}, feed samples $\calS^{\s{R}}$ and $\calS^{\s{F}}$ from the Markov chains are supplied to the i.i.d. tester in Algorithm~\ref{algo:tester_HT_4}. These samples are pulled using the mapping $\psi$, and thus are guaranteed i.i.d., as mentioned right after Lemma~\ref{$m$-joint-$k$-cover time upper bound}. The sample condition in \eqref{eqn:sampleGurantee} guarantees that all states are visited ``reasonable" number of times, jointly by all the $\s{M}$ chains, and for all users. Accordingly, we can apply an i.i.d. identity tester to each state's conditional distribution, and the auditing procedure return ``\texttt{YES}" if this distribution passes its corresponding i.i.d. test.

At this point we would like to mention that our auditing procedure is not required to be disclosed to the internal AF mechanism used by the platform, which may not consent to be shared. This provides also a flexibility in regulating the model with no need for adaptation with respect to any future modification of the internal AF. Furthermore, our procedure can be applied using only access to users' observations (their feeds) in order to infer the influence of the platforms on their beliefs, decision-making, and ultimately on their actions, while having no access to their actual beliefs. It is clear that this way no further privacy leakage is incurred from the auditing process.\footnote{While data hacking remains a possibility, it would not be considered a regulatory flaw, as it could occur regardless of regulation. The involved parties are the platform and auditor, where the platform possesses data access and the auditor utilizes data for testing, thus maintaining user data privacy with sensible precautions.} The bound we derived on $m$-joint $k$-cover time in Subsection~\ref{sec:Auxiliary Definitions and Lemmas} gives a simpler sample complexity bound for the auditing procedure. In particular, using Lemma~\ref{$m$-joint-$k$-cover time upper bound}, we get that the number of samples, per user, can be bounded as,
\begin{align}
\s{T}=\calO_\delta\p{\max_{u\in[\s{U}]}\max_{\s{W}\in\{\s{Q}_u^{\s{F}},\s{P}_u^{\s{R}}\}}\pp{t_{\s{mix}}(\s{W})\vee\frac{t_{\s{cov}}(\s{W})\log |\calX|}{\s{M}}}\log |\calX|+\frac{\bar{m}\log |\calX|}{\pi_{\star}(\s{W})}},
\end{align}
where $\pi_{\star}(\s{W})$ denotes the minimum stationary distribution of the Markov chain with transition probability matrix $\s{W}$, and $\mathcal{O}_{\delta}$ hides logarithmic factors in $\delta$.

\subsection{Counterfactual regulation}
\label{subsec:Counterfactual}

Above, we have focused on the ``filtered vs. reference" feeds approach. However, it is clear that other frameworks can be formulated. Consider the following as an alternative. Let $\calS$ be a \emph{regulatory statement} that an inspector (or, perhaps, the platform itself) wish to test. For example, $\calS$ could be: ``\emph{The platform should produce similar articles for users who are identical except for property $\mathscr{P}$}", where $\mathscr{P}$ could be ethnicity, sexual orientation, gender, a combination of these factors, etc. Let $\calU_{\mathscr{P}}\subset\calV\times\calV$ be a subset of pairs of users that comply with $\mathscr{P}$. Then, for any pair of users $(i,j)\in\calU_{\mathscr{P}}$, the inspector's objective is to determine whether the platform's filtering algorithm cause user $i$'s and user $j$'s beliefs and actions be significantly different. A similar approach, was studied recently in \cite{cen2020regulating} under a time-independent i.i.d. model. We take into account the inherent dependency on the time dimension as in ``real-world" applications regulations must be enforced over time, as explained in the introduction. Similarly to Subsection~\ref{subsec:regulatorgoal}, we define the notion of counterfactual violation as follows. 

\begin{definition}[Counterfactual total variability]
Let $\calU_{\mathscr{P}}\subset[\s{U}]\times[\s{U}]$ be a subset of pairs of users that comply with $\mathscr{P}$. Then, for any pair of users $(i,j)\in\calU_{\mathscr{P}}$, the total variability in algorithmic filtering behavior for counterfactual users is given by
\begin{align}
\bar{\mathds{V}}_{\s{cu}}(\calS,\calU_{\mathscr{P}})&\triangleq\frac{1}{|\calU_{\mathscr{P}}|}\sum_{(i,j)\in\calU_{\mathscr{P}}}\max _{\ell\in\calX}\s{d}_{\s{TV}}\p{Q_{i}(\ell,\cdot),Q_{j}(\ell,\cdot)}\\
&=\frac{1}{|\calU_{\mathscr{P}}|}\sum_{(i,j)\in\calU_{\mathscr{P}}}\max _{\ell\in\calX} \left\|\mathbf{Q}_{i}(\ell)- \mathbf{Q}_{j}(\ell) \right\|_{1}
\\
&= \frac{1}{|\calU_{\mathscr{P}}|}\sum_{(i,j)\in\calU_{\mathscr{P}}}\left\|\s{Q}^{\f{F}}_{i}-\s{Q}^{\f{F}}_{j}\right\|_{\infty}.\label{eqn:Vcu}
\end{align}
\end{definition}
Then, in the same spirit of the previous subsection, we define the investigator's task to test for violations in the following sense:
\begin{align}
\calH^{\calS}_{0}:\bar{\mathds{V}}_{\s{cu}}(\calS,\calU_{\mathscr{P}})\leq\varepsilon_1\;\quad\s{vs.}\quad\calH^{\calS}_{1}:\bar{\mathds{V}}_{\s{cu}}(\calS,\calU_{\mathscr{P}})\geq\varepsilon_2.\label{eqn:HT-CR}
\end{align}
As before, the goal here is to construct good inspection procedures given only $\calS$ and a black-box access to the filtering algorithm. Note also that $\calU_{\mathscr{P}}$ need not correspond to real users and could represent hypothetical users. 
Now, comparing \eqref{eqn:VfilterObjective} and \eqref{eqn:HTbelief3} with \eqref{eqn:Vcu} and \eqref{eqn:HT-CR}, it is clear that the hypothesis test in \eqref{eqn:HT-CR} is the same as the one in \eqref{eqn:HTbelief3}, if each pair of filtered and reference distributions that correspond to some user is replaced with a pair of filtered distributions that correspond to a pair of users in $\calU_{\mathscr{P}}$. Accordingly, consider the counterfactual auditing procedure that appears in Algorithm~\ref{algo:TolerantClosenessTestingRegu2}. It is clear that the underlying idea in Algorithm~\ref{algo:TolerantClosenessTestingRegu2} is the same as the one in Algorithm~\ref{algo:TolerantClosenessTestingRegu}. The following is a direct consequence of Theorem~\ref{lemma:m-Jointly-k-cover time and tolerant-closeness-testing Markov chains}. 
\begin{theorem}[Sample complexity]
Given an $(\varepsilon_1,\varepsilon_2, \delta)$ i.i.d. tolerant-closeness-tester for $n$-state distributions with sample complexity $m(n,\varepsilon_{1}, \varepsilon_{2},\delta)$, then we can $(\varepsilon_1,\varepsilon_2, \delta)$ testing hypothesis (\ref{eqn:HT-CR}) using,
\begin{align}
    \s{T} = \mathcal{O}\left(\max_{(u,v)\in\calU_{\mathscr{P}}}\max_{\s{W}\in\{\s{Q}_u^{\f{F}},\s{Q}_v^{\f{F}}\}} t_{\s{cov}}^{\tilde{m}}\left(\s{M};\s{W}\right)\log\frac{|\calU_{\mathscr{P}}|}{\delta}\right),
\end{align}
samples for each pair of users in $\calU_{\mathscr{P}}$. 
\label{Sample_complexity_Time_Dependent_Counter_Reg_Procedure}
\end{theorem}
It should be mentioned that the auditing procedure requires a black-box access to the filtering algorithm only, and the internal filtering mechanism is oblivious to the auditor (SMPs will not grant auditors full access to their filtering algorithm). This in turn also implies that auditing procedure can work even if the filtering algorithm changes over time. Finally, as in the previous subsection, the bounds we derived on $m$-joint$k$-cover time in Subsection~\ref{sec:Auxiliary Definitions and Lemmas} give simpler sample complexity bounds. Indeed, the number of samples, for each pair of users in $\calU_{\mathscr{P}}$, can be written as,
\begin{align}
\s{T}=\calO_\delta\p{\max_{(u,v)\in\calU_{\mathscr{P}}}\max_{\s{W}\in\{\s{Q}_u^{\f{F}},\s{Q}_v^{\f{F}}\}}\pp{t_{\s{mix}}(\s{W})\vee\frac{t_{\s{cov}}(\s{W})\log |\calX|}{\s{M}}}\log |\calX|+\frac{\bar{m}\log |\calX|}{\pi_{\star}(\s{W})}}.
\end{align}

\begin{algorithm}[t]
\KwIn{$\s{T}$, $n\triangleq|\calX|$, $\varepsilon_1,\varepsilon_1,\delta$, $\bar{m}$, and feeds $\{\bX^{\s{F}}_{u}(t),\bX^{\s{F}}_{v}(t)\}_{t=1}^{\s{T}}$, for every $(u,v)\in\calU_{\mathscr{P}}$.}
\KwOut{ \texttt{YES} if $\bar{\mathds{V}}_{\s{cu}}(\calS,\calU_{\mathscr{P}}) \leq \varepsilon_1$ / \texttt{NO} if $\bar{\mathds{V}}_{\s{cu}}(\calS,\calU_{\mathscr{P}}) \geq \varepsilon_2$.}
{\bf{For}} $i \leftarrow 1,2 \ldots\ldots,n $ \\
$\hspace{0.75cm}$Set $\calS\leftarrow \emptyset$ and $\tilde{\calS} \leftarrow \emptyset$\\
$\hspace{0.75cm}${\bf{For}} every pair  $(u,v)\in\calU_{\mathscr{P}}$ \\
 $\hspace{1.5cm}$ {\bf{If }}  $\sum_{j=1}^{\s{M}}\calN_{i}^{{\f{x}}_{j,u}^{\s{F}}}<\bar{m}$ {\bf{or}}  $\sum_{j=1}^{\s{M}}\calN_{i}^{{\f{x}}_{j,v}^{\s{F}}}<\bar{m}$ \\
$\hspace{2.25cm}$  {\bf{Return}} \texttt{NO}\\
 $\hspace{1.5cm}$ Calculate $\calS_u\leftarrow\cup_{j=1}^{\s{M}}\psi_{\bar{m}}^{(i)}\left(\{{\f{x}}_{j,u}^{\s{F}}(t)\}_{t=1}^{\s{T}}\right)$ and  $\tilde{\calS}_v\leftarrow\cup_{j=1}^{\s{M}}\psi_{\bar{m}}^{(i)}\left(\{{\f{x}}_{j,v}^{\s{F}}(t)\}_{t=1}^{\s{T}}\right)$ \\
 $\hspace{1.5cm}$ Do $\calS \leftarrow \calS\cup \calS_u$ and $\tilde{\calS} \leftarrow \tilde{\calS}\cup \tilde{\calS}_v$ \\
 $\hspace{0.75cm}$ {\bf{If }} $\textsc{IIDTester}(\calS,\tilde{\calS},\delta,\varepsilon_1,\varepsilon_2,\bar{m},n)=$ \texttt{NO}\\
     $\hspace{2cm}$  {\bf{Return}} \texttt{NO}\\
{\bf{Return}} \texttt{YES}
\caption{Counterfactual auditing procedure}
\label{algo:TolerantClosenessTestingRegu2}
\end{algorithm}

\section{Proofs}
This section is devoted to the proofs of our results.

\subsection{Proof of Theorem \ref{theorem:iid_sample_complexity}} 

In this subsection we prove Theorem \ref{theorem:iid_sample_complexity}. To this end, we start by proving a few auxiliary results which characterize the first and second order statistics of the count in \eqref{eqn:G_est}.

\subsubsection{Auxiliary Results}

\begin{lemma}\label{lem:9}
Let $\delta_1\in(0,1)$, and recall the definitions in \eqref{fuidef}--\eqref{eqn:G_est}. Then, there exist absolute constants $c_1,c_2,c_3>0$, such that the following hold with probability at least $1-\delta_1$,
\begin{align}
\mathbb{E}\left[\left.G \vphantom{\hat{f}_{u,i}}\right| \hat{f}_{u,i},u\in[\s{U}],i\in[n]\right] \geq \frac{\delta_1m^{2}\left(\sum_{u=1}^{\s{U}}\left\|P_{u} - Q_{u}\right\|_{1}\right)^{2}}{ c_{1}\sum_{u=1}^{\s{U}}\sum_{i=1}^{n} f_{u,i}},\label{eqn:lem91}
\end{align}
and
\begin{align}
    \mathbb{E}\left[\left.G \vphantom{\hat{f}_{u,i}}\right| \hat{f}_{u,i},u\in[\s{U}],i\in[n]\right] \leq  \frac{c_{2}}{\delta_1} \sum_{u=1}^{\s{U}}\sum_{i=1}^{n} \frac{m^{2}\left(P_{u,i} - Q_{u,i}\right)^{2}}{f_{u,i}},\label{eqn:lem92}
\end{align}
and
\begin{align}
   \s{Var}\left[\left.G \vphantom{\hat{f}_{u,i}}\right| \hat{f}_{u,i},u\in[\s{U}],i\in[n]\right]\leq \frac{c_{3}}{\delta_1} \sum_{u=1}^{\s{U}}\sum_{i=1}^{n} \frac{\s{Var}\left[G_{u,i}\right]}{f_{u,i}^2}.\label{eqn:lem93}
\end{align}    
\end{lemma}

\begin{proof}[Proof of Lemma~\ref{lem:9}] By standard properties of the Poisson distribution, the random variables in the definition of $G_{u,i}$ are statistically independent. Therefore, 
\begin{align}\EE[G_{u,i}]&= \EE[(V_{u,i}-Y_{u,i})^2-V_{u,i}-Y_{u,i}] \nonumber \\
&=\EE[V_{u,i}^2] -2\EE[V_{u,i}]\EE[Y_{u,i}] + \EE[Y_{u,i}^2] - \EE[V_{u,i}] - \EE[Y_{u,i}]\nonumber \\
&=(m P_{u,i})^2 +mP_{u,i} -2m^2P_{u,i}Q_{u,i} + (m Q_{u,i})^2 + m Q_{u,i}- m^2P_{u,i} - m Q_{u,i}\nonumber \\
&=(m P_{u,i})^2 -2m^2P_{u,i}Q_{u,i} + (m Q_{u,i})^2= m^2(P_{u,i}-Q_{u,i})^2\nonumber \\
& = m^2\left|P_{u,i}-Q_{u,i}\right|^2.
\label{E of G}
\end{align}
Hence, $G_{u,i}$ is an unbiased estimator of $m^2|P_{u,i} - Q_{u,i}|^2$. Similarly,
\begin{align}\Var(G_{u,i})&=\Var[(V_{u,i}-Y_{u,i})^2-V_{u,i}-Y_{u,i}]\nonumber \\
&=\EE[((V_{u,i}-Y_{u,i})^2-V_{u,i}-Y_{u,i})^2] - \EE[((V_{u,i}-Y_{u,i})^2-V_{u,i}-Y_{u,i})]^2\nonumber \\
&=\EE[((V_{u,i}-Y_{u,i})^4-2(V_{u,i}-Y_{u,i})^3+(V_{u,i}-Y_{u,i})^2]  \nonumber \\
&\ \ \ \ \ \ \ -\EE[((V_{u,i}-Y_{u,i})^2-V_{u,i}-Y_{u,i})]^2\nonumber \\
&=4m^3(P_{u,i}-Q_{u,i})^2(P_{u,i}+Q_{u,i})+2m^2(P_{u,i}+Q_{u,i}).
\label{Var of Guti}
\end{align}
Next, using the fact that $G_{u,i}$ and $\hat{f}_{u,i}$ are independent, by the linearity of the expectation, we obtain that the conditional expectation of $G$ is,
\begin{align}\EE\left[\left.G \vphantom{\hat{f}_{u,i}}\right|\hat{f}_{u,i},u\in[\s{U}],i\in[n]\right]=&\EE\left[\sum_{u=1}^{\s{U}}\sum_{i=1}^{n}\frac{G_{u,i}}{\hat{f}_{u,i}}|\hat{f}_{u,i},u\in[\s{U}],i\in[n]\right] \\
&=\sum_{u=1}^{\s{U}}\sum_{i=1}^{n}\frac{\EE\left[G_{u,i}\right]}{\hat{f}_{u,i}}. \label{conditional expectation of G}
\end{align}
Similarly, the conditional variance of $G$ is,
\begin{align}\Var\left[\left.G \vphantom{\hat{f}_{u,i}}\right|\hat{f}_{u,i},u\in[\s{U}],i\in[n]\right]&=\Var\left[\sum_{u=1}^{\s{U}}\sum_{i=1}^{n}\frac{G_{u,i}}{\hat{f}_{u,i}} |\hat{f}_{u,i},u\in[\s{U}],i\in[n]\right]\\
&=\sum_{u=1}^{\s{U}}\sum_{i=1}^{n}\frac{\Var\left[G_{u,i}\right]}{\hat{f}_{u,i}^2}. \label{conditional variance of G}
\end{align}
Combining \eqref{E of G} and \eqref{conditional expectation of G} we get,
\begin{align}
\mathbb{E}\left[\left.G \vphantom{\hat{f}_{u,i}}\right|\hat{f}_{u,i},u\in[\s{U}],i\in[n]\right]&=\sum_{u=1}^{\s{U}}\sum_{i=1}^{n}\frac{\mathbb{E}\left[G_{u,i}\right]}{\hat{f}_{u,i}}\\
&=\sum_{u=1}^{\s{U}}\sum_{i=1}^{n} \frac{m^{2}\left(P_{u,i} - Q_{u,i}\right)^{2}}{\hat{f}_{u,i}}\\
&\geq\frac{m^{2}\left(\sum_{u=1}^{\s{U}}\sum_{i=1}^{n}\left|P_{u,i} - Q_{u,i}\right|\right)^{2}}{\sum_{u=1}^{\s{U}}\sum_{i=1}^{n} \hat{f}_{u,i}}\\
&\geq\frac{m^{2}\left(\sum_{u=1}^{\s{U}}\left\|P_{u} - Q_{u}\right\|_1\right)^{2}}{\sum_{u=1}^{\s{U}}\sum_{i=1}^{n} \hat{f}_{u,i}},
\end{align}
where the first inequality follows from the following fact that for any sequence of real-valued numbers $\{a_{i}\}_{i=1}^{n}$ and positive  real-valued numbers $\left\{b_{i}\right\}_{i=1}^{n}$, we have,
\begin{align}
\sum_{i=1}^{n} \frac{a_{i}^{2}}{b_{i}} \geq \frac{\left(\sum_{i=1}^{n}\left|a_{i}\right|\right)^{2}}{\sum_{i=1}^{n} b_{i}},
\end{align}
and the last inequality follows by applying Cauchy-Schwarz to, \begin{align}
\sum_{i=1}^{n}\left|a_{i}\right|=\sum_{i=1}^{n} \frac{\sqrt{b_{i}}\left|a_{i}\right|}{\sqrt{b_{i}}}.
\end{align}
Next, Lemma 2.5 in \cite{canonne2021price} states that there exist absolute constants $c_{1}, c_{2}, c_{3}>0$ such that, for every $u\in[\s{U}],i \in[n]$, we have $\mathbb{E}[ \hat{f}_{u,i}] \leq c_{1}f_{u,i}$, $\mathbb{E}[ \hat{f}_{u,i}^{-1}] \leq \frac{c_{2}}{ f_{u,i}}$, and $\mathbb{E}[ \hat{f}_{u,i}^{-2}] \leq \frac{c_{3}}{ {f^{2}_{u,i}}}$. Moreover, by definition, the random random variables $\hat{f}_{u,i}$ are non-negative, and thus, applying by Markov's inequality, we obtain that, 
\begin{align}
\sum_{i=1}^{n} \hat{f}_{u,i} \leq\frac{1}{\delta_1} \sum_{i=1}^{n} \mathbb{E}\left[\hat{f}_{u,i}\right],
\end{align}
with probability at least $1-\delta_1$, for any $u\in[\s{U}]$. Combined with Lemma 2.5 in \cite{canonne2021price}, this means that, with probability at least $1-\delta_1$,
\begin{align}
\mathbb{E}\left[\left.G \vphantom{\hat{f}_{u,i}}\right|\hat{f}_{u,i},u\in[\s{U}],i\in[n]\right] \geq \frac{\delta_1m^{2}\left(\sum_{u=1}^{\s{U}}\left\|P_{u} - Q_{u}\right\|_{1}\right)^{2}}{ c_{1}\sum_{u=1}^{\s{U}}\sum_{i=1}^{n} f_{u,i}}.
\label{first bound}
\end{align}
Next, applying Markov's inequality for the non-negative random variable $\mathbb{E}\left[\left.G \vphantom{\hat{f}_{u,i}}\right|\hat{f}_{u,i},u\in[\s{U}],i\in[n]\right]$, along with Lemma 2.5 in \cite{canonne2021price}, we obtain with probability at least $1-\delta_1$,
\begin{align}
\mathbb{E}\left[\left.G \vphantom{\hat{f}_{u,i}}\right|\hat{f}_{u,i},u\in[\s{U}],i\in[n]\right] &\leq \frac{1}{\delta_1} \mathbb{E}\left[\mathbb{E}\left[\left.G \vphantom{\hat{f}_{u,i}}\right|\hat{f}_{u,i},u\in[\s{U}],i\in[n]\right]\right]\\
&=\frac{1}{\delta_1} \mathbb{E}\left[\sum_{u=1}^{\s{U}}\sum_{i=1}^{n}\frac{\mathbb{E}\left[G_{u,i}\right]}{\hat{f}_{u,i}}\right] \\
&\leq  \frac{c_{2}}{\delta_1} \sum_{u=1}^{\s{U}}\sum_{i=1}^{n} \frac{m^{2}\left(P_{u,i} - Q_{u,i}\right)^{2}}{f_{u,i}}.
\label{second bound}
\end{align}
Similarly, with probability at least $1-\delta_1$,
\begin{align}
\s{Var}\left[\left.G \vphantom{\hat{f}_{u,i}}\right|\hat{f}_{u,i},u\in[\s{U}],i\in[n]\right] &\leq \frac{1}{\delta_1} \mathbb{E}\left[\s{Var}\left[\left.G \vphantom{\hat{f}_{u,i}}\right|\hat{f}_{u,i},u\in[\s{U}],i\in[n]\right]\right]\nonumber\\
&=\frac{1}{\delta_1} \mathbb{E}\left[\sum_{u=1}^{\s{U}}\sum_{i=1}^{n}\frac{\s{Var}\left[G_{u,i}\right]}{\hat{f}_{u,i}^2}\right] \\
&\leq \frac{c_{3}}{\delta_1} \sum_{u=1}^{\s{U}}\sum_{i=1}^{n} \frac{\s{Var}\left[G_{u,i}\right]}{f_{u,i}^2}.
\label{thired bound}
\end{align}
This concludes the proof.
\end{proof}

By the union bound, we conclude that \eqref{eqn:lem91}--\eqref{eqn:lem93}, hold simultaneously with probability at least $1-3\delta_1$. We next bound the terms in the right-hand-side of \eqref{eqn:lem91}--\eqref{eqn:lem93}, separately for $m \geq n$ and $m \leq n$, respectively. We follow similar ideas as in \cite[Lemma 2.3]{canonne2021price} and \cite[Lemma 2.4]{canonne2021price}. We have the following result.

\begin{lemma} For $m \geq n$, the following hold,
\begin{align}
    \sum_{u=1}^{\s{U}}\sum_{i=1}^{n} \frac{\s{Var}\left[G_{u,i}\right]}{f_{u,i}^2} &\leq \frac{10\s{U}m^{2}}{n},\label{eqn:firstIneq}\\
    \sum_{i=1}^{n} \frac{m^2(P_{u,i}-Q_{u,i})^2}{f_{u,i}}&\leq \frac{m^{3 / 2}\|P_{u}-Q_{u}\|_{1}}{n^{\frac{1}{2}}},\label{eqn:secondIneq}
\end{align}
and
\begin{align}
    \frac{m^{2}\left(\sum_{u=1}^{\s{U}}\left\|P_{u} - Q_{u}\right\|_{1}\right)^{2}}{\sum_{u=1}^{\s{U}}\sum_{i=1}^{n} f_{u,i}}  \geq \min \left(\frac{m^{3 / 2} \sum_{u=1}^{\s{U}}\left\|P_{u} - Q_{u}\right\|_{1} }{2 \left(\s{U}n\right)^{\frac{1}{2}}}, \frac{m^{2}\left(\sum_{u=1}^{\s{U}}\left\|P_{u} - Q_{u}\right\|_{1}\right)^{2}}{6\s{U}n}\right).\label{eqn:thirdIneq}
\end{align}
\label{utlizem>n}
\end{lemma}

\begin{proof}[Proof of Lemma \ref{utlizem>n}] We start by proving \eqref{eqn:firstIneq}. From \eqref{Var of Guti}, we get
\begin{align}
\sum_{u=1}^{\s{U}}\sum_{i=1}^{n} \frac{\s{Var}\left[G_{u,i}\right]}{f_{u,i}^2} &=\sum_{u=1}^{\s{U}}\sum_{i=1}^{n} \frac{4m^3(P_{u,i}-Q_{u,i})^2(P_{u,i}+Q_{u,i})+2m^2(P_{u,i}+Q_{u,i})^{2}}{f_{u,i}^{2}} \\
&=\sum_{u=1}^{\s{U}}\sum_{i=1}^{n} \frac{4 m^{3}\left(P_{u,i}-Q_{u,i}\right)^{2}\left(P_{u,i}+Q_{u,i}\right)+2 m^{2}\left(P_{u,i}+Q_{u,i}\right)^{2}}{\left(\max \left\{\sqrt{m n} \cdot\left|P_{u,i}-Q_{u,i}\right|, n \cdot\left(P_{u,i}+Q_{u,i}\right), 1\right\}\right)^{2}}\label{eqn:deno} \\
& \leq \sum_{u=1}^{\s{U}}\sum_{i=1}^{n} \frac{4 m^{3}\left(P_{u,i}+Q_{u,i}\right)}{m n}+\sum_{u=1}^{\s{U}}\sum_{i=1}^{n} \frac{2 m^{2}}{n^{2}}\\
&=\frac{10\s{U}m^{2}}{n},
\end{align}
where the inequality follows by lower bounding the denominator in \eqref{eqn:deno} by $mn(P_{u,i}-Q_{u,i})^2$ for the first term in the numerator, and by $n^2(P_{u,i}+Q_{u,i})^2$ for the second term in the numerator. Next, we prove \eqref{eqn:secondIneq}. We have,
\begin{align}
 \sum_{i=1}^{n} \frac{m^2(P_{u,i}-Q_{u,i})^2}{f_{u,i}} &= \sum_{i=1}^{n} \frac{m^{2}\left|P_{u,i}-Q_{u,i}\right|^{2}}{\max \left\{\sqrt{m n} \cdot\left|P_{u,i}-Q_{u,i}\right|,n \cdot\left(P_{u,i}+Q_{u,i}\right), 1\right\}} \\
& \leq  \sum_{i=1}^{n} \frac{m^{3 / 2}\left|P_{u,i}-Q_{u,i}\right|}{n^{\frac{1}{2}}}\\
&=\frac{m^{3 / 2}\|P_{u}-Q_{u}\|_{1}}{n^{\frac{1}{2}}}.
\end{align}
Finally, we prove \eqref{eqn:thirdIneq}. Note that,
\begin{align}
\frac{m^{2}\left(\sum_{u=1}^{\s{U}}\left\|P_{u} - Q_{u}\right\|_{1}\right)^{2}}{\sum_{u=1}^{\s{U}}\sum_{i=1}^{n} f_{u,i}}&=\frac{m^{2}\left(\sum_{u=1}^{\s{U}}\left\|P_{u} - Q_{u}\right\|_{1}\right)^{2}}{\sum_{u=1}^{\s{U}}\sum_{i=1}^{n} \max \left\{\sqrt{m n} \cdot\left|P_{u,i}-Q_{u,i}\right|,n \cdot\left(P_{u,i}+Q_{u,i}\right), 1\right\}} \nonumber\\
& \geq \frac{m^{2}\left(\sum_{u=1}^{\s{U}}\left\|P_{u} - Q_{u}\right\|_{1}\right)^{2}}{\sum_{u=1}^{\s{U}}\sum_{i=1}^{n}\left(\sqrt{m n} \cdot\left|P_{u,i}-Q_{u,i}\right|+n \cdot\left(P_{u,i}+Q_{u,i}\right)+1\right)}\\
&=\frac{m^{2}\left(\sum_{u=1}^{\s{U}}\left\|P_{u} - Q_{u}\right\|_{1}\right)^{2}}{\sqrt{m n} \cdot\sum_{u=1}^{\s{U}}\left\|P_{u} - Q_{u}\right\|_{1}+2 \s{U}n+\s{U}n} \\
&\geq \min \left(\frac{m^{3 / 2} \sum_{u=1}^{\s{U}}\left\|P_{u} - Q_{u}\right\|_{1} }{2 \s{U}\sqrt{n}}, \frac{m^{2}\left(\sum_{u=1}^{\s{U}}\left\|P_{u} - Q_{u}\right\|_{1}\right)^{2}}{6\s{U}n}\right),
\end{align}
where the first inequality follows by the trivial bound $\max(a,b)\leq a+b$, for any two non-negative numbers $a$ and $b$.
\end{proof} 

Applying Lemma~\ref{utlizem>n} on \eqref{eqn:lem91}--\eqref{eqn:lem93}, we obtain the following corollary for $m \geq n$.
\begin{corollary}\label{cor:1}
   For $m \geq n$, the following hold with probability at least $1-\delta_1$,
   \begin{align}
       &\mathbb{E}\left[\left.G \vphantom{\hat{f}_{u,i}}\right|\hat{f}_{u,i},u\in[\s{U}],i\in[n]\right]\nonumber\\
       &\quad\quad\geq \frac{\delta_1}{c_1}\min \left(\frac{m^{3 / 2} \sum_{u=1}^{\s{U}}\left\|P_{u} - Q_{u}\right\|_{1} }{2 \s{U}\sqrt{n}}, \frac{m^{2}\left(\sum_{u=1}^{\s{U}}\left\|P_{u} - Q_{u}\right\|_{1}\right)^{2}}{6\s{U}n}\right),\\
       &\mathbb{E}\left[\left.G \vphantom{\hat{f}_{u,i}}\right|\hat{f}_{u,i},u\in[\s{U}],i\in[n]\right]\leq \frac{c_{2}}{\delta_1}\sum_{u=1}^{\s{U}} \frac{m^{3 / 2}\|P_{u}-Q_{u}\|_{1}}{n^{\frac{1}{2}}},
   \end{align}
   and
   \begin{align}
      \s{Var}\left[\left.G \vphantom{\hat{f}_{u,i}}\right|\hat{f}_{u,i},u\in[\s{U}],i\in[n]\right]  \leq \frac{c_{3}}{\delta_1}\frac{10\s{U}m^{2}}{n}.
   \end{align}
\end{corollary}

Next, we move froward to the case where $m \leq n$. We have the following result.
\begin{lemma} For $m \leq n$, the following hold,
\begin{align}
    &\sum_{i=1}^{n} \frac{\s{Var}\left(G_{u,i}\right)}{f_{u,i}^{2}} \leq 24 m,\label{eqn:firstIneq11}\\
    &\sum_{i=1}^{n} \frac{m^{2}\left(P_{u,i}-Q_{u,i}\right)^{2}}{f_{u,i}} \leq m\|P_{u}-Q_{u}\|_{1},\label{eqn:firstIneq12}
\end{align}
and
\begin{align}
    \frac{m^{2}\left(\sum_{u=1}^{\s{U}}\left\|P_{u} - Q_{u}\right\|_{1}\right)^{2}}{\sum_{u=1}^{\s{U}}\sum_{i=1}^{n}f_{u,i}} \geq \frac{m^{2}\left(\sum_{u=1}^{\s{U}}\left\|P_{u} - Q_{u}\right\|_{1}\right)^{2}}{3 n}.\label{eqn:firstIneq13}
\end{align}
\label{utlizem<n}
\end{lemma}

\begin{proof}[Proof of Lemma \ref{utlizem<n}]
As before, we start by proving \eqref{eqn:firstIneq11}. We have,
\begin{align}
\sum_{i=1}^{n} \frac{\s{Var}\left(G_{u,i}\right)}{f_{u,i}^{2}} &=\sum_{i=1}^{n} \frac{4m^3(P_{u,i}-Q_{u,i})^2(P_{u,i}+Q_{u,i})+2m^2(P_{u,i}+Q_{u,i})^{2}}{f_{u,i}^{2}} \\
&=\sum_{i=1}^{n} \frac{4m^3(P_{u,i}-Q_{u,i})^2(P_{u,i}+Q_{u,i})+2m^2(P_{u,i}+Q_{u,i})^{2}}{\left(\max \left\{m \cdot\left(P_{u,i}+Q_{u,i}\right), 1\right\}\right)^{2}} \\
& \leq \sum_{i=1}^{n} \frac{4 m^{3}\left(P_{u,i}-Q_{u,i}\right)^{2}\left(P_{u,i}+Q_{u,i}\right)+4 m^{2}\left(P_{u,i}-Q_{u,i}\right)^{2}+8 m^{2} Q_{u,i}^{2}}{\max \left\{m^{2} \cdot\left(P_{u,i}+Q_{u,i}\right)^{2}, 1\right\}}\\
& \leq \sum_{i=1}^{n} \frac{4 m^{3}\left(P_{u,i}-Q_{u,i}\right)^{2}\left(P_{u,i}+Q_{u,i}\right)}{m^{2} \cdot\left(P_{u,i}+Q_{u,i}\right)^{2}}+\sum_{i=1}^{n} \frac{4 m^{2}\left(P_{u,i}-Q_{u,i}\right)^{2}}{m \cdot\left(P_{u,i}+Q_{u,i}\right)}\nonumber\\
&\ \ \ \ +\sum_{i=1}^{n} \frac{8 m^{2} Q_{u,i}^{2}}{\max \left\{m^{2}\left(P_{u,i}+Q_{u,i}\right)^{2}, 1\right\}} \\
&\leq 4 m \sum_{i=1}^{n}\left|P_{u,i}-Q_{u,i}\right|+4 m \sum_{i=1}^{n}\left|P_{u,i}-Q_{u,i}\right|+\sum_{i=1}^{n} \frac{8 m^{2} Q_{u,i}^{2}}{\max \left\{m\left(P_{u,i}+Q_{u,i}\right), 1\right\}}\\
& \leq 8 m\|P_{u}+Q_{u}\|_{1}+\sum_{i=1}^{n} 8 m Q_{u,i} \\
& \leq 24 m,
\end{align}
where the first inequality follows from the fact that $(a+b)^2\leq 2(a-b)^2+4b^2$, for any $a,b\geq0$, the second inequality follows by lower bounding the denominator by individual terms in the maximum, and the third inequality follows from the trivial bound $\frac{|a-b|}{a+b}\leq1$, for $a,b\geq0$. Next, for \eqref{eqn:firstIneq12}, we note that,
\begin{align}
\sum_{i=1}^{n} \frac{m^{2}\left(P_{u,i}-Q_{u,i}\right)^{2}}{{f_{u,i}}}&=\sum_{i=1}^{n} \frac{m^{2}\left|P_{u,i}-Q_{u,i}\right|^{2}}{\max \left\{m \cdot\left(P_{u,i}+Q_{u,i}\right), 1\right\}}\\ & \leq \sum_{i=1}^{n} m\left|P_{u,i}-Q_{u,i}\right|=m\|P_{u}-Q_{u}\|_{1}.
\end{align}
Finally, we prove \eqref{eqn:firstIneq13}. We have,
\begin{align}
\frac{m^{2}\left(\sum_{u=1}^{\s{U}}\left\|P_{u} - Q_{u}\right\|_{1}\right)^{2}}{\sum_{u=1}^{\s{U}}\sum_{i=1}^{n}f_{u,i}}&=\frac{m^{2}\left(\sum_{u=1}^{\s{U}}\left\|P_{u} - Q_{u}\right\|_{1}\right)^{2}}{\sum_{u=1}^{\s{U}}\sum_{i=1}^{n} \max \left\{m \cdot\left(P_{u,i}+Q_{u,i}\right), 1\right\}}\\ 
&\geq\frac{m^{2}\left(\sum_{u=1}^{\s{U}}\left\|P_{u} - Q_{u}\right\|_{1}\right)^{2}}{\sum_{u=1}^{\s{U}}\left(m \cdot\left\|P_{u}+Q_{u}\right\|_1+L\right)}\\
&=\frac{m^{2}\left(\sum_{u=1}^{\s{U}}\left\|P_{u} - Q_{u}\right\|_{1}\right)^{2}}{2 m+n}\\ &\geq\frac{m^{2}\left(\sum_{u=1}^{\s{U}}\left\|P_{u} - Q_{u}\right\|_{1}\right)^{2}}{3 n}.
\end{align}
This concludes the proof.
\end{proof} 

Applying Lemma~\ref{utlizem<n} on \eqref{eqn:lem91}--\eqref{eqn:lem93}, we obtain the following corollary for $m \leq n$.
\begin{corollary}\label{cor:2}
    For $m \leq n$, the following hold with probability at least $1-\delta_1$,
    \begin{align}
        &\mathbb{E}\left[\left.G \vphantom{\hat{f}_{u,i}}\right|\hat{f}_{u,i},u\in[\s{U}],i\in[n]\right] \geq \frac{\delta_1}{c1}\frac{m^{2}\left(\sum_{u=1}^{\s{U}}\left\|P_{u} - Q_{u}\right\|_{1}\right)^{2}}{3 n},\label{eqn:Item11}\\
        &\mathbb{E}\left[\left.G \vphantom{\hat{f}_{u,i}}\right|\hat{f}_{u,i},u\in[\s{U}],i\in[n]\right]\leq \frac{c_{2}}{\delta_1}\sum_{u=1}^{\s{U}}  m\|P_{u}-Q_{u}\|_{1},\label{eqn:Item12}\\
        &\s{Var}\left[\left.G \vphantom{\hat{f}_{u,i}}\right|\hat{f}_{u,i},u\in[\s{U}],i\in[n]\right]  \leq  \frac{24m\s{U}c_{3}}{\delta_1},\label{eqn:Item13}
    \end{align}
    and
    \begin{align}
   \s{Var}\left[\left.G \vphantom{\hat{f}_{u,i}}\right|\hat{f}_{u,i},u\in[\s{U}],i\in[n]\right]  &\leq \frac{1}{40}\left(\mathbb{E}\left[\left.G \vphantom{\hat{f}_{u,i}}\right|\hat{f}_{u,i},u\in[\s{U}],i\in[n]\right]\right)^{2}\nonumber\\
    &\ +324 \mathbb{E}\left[\left.G \vphantom{\hat{f}_{u,i}}\right|\hat{f}_{u,i},u\in[\s{U}],i\in[n]\right]+648 m^{2}\sum_{u=1}^{\s{U}}\|Q_{u}\|_{2}^{2}.\label{eqn:Item14}
\end{align}
\end{corollary}

\begin{proof}[Proof of Corollary~\ref{cor:2}]
Inequalities \eqref{eqn:Item11}--\eqref{eqn:Item13} follow almost directly from Lemma~\ref{utlizem<n}, and we next focus on \eqref{eqn:Item14}. First, note that
\begin{align}
\s{Var}\left[\left.G \vphantom{\hat{f}_{u,i}}\right|\hat{f}_{u,i}  \right] &=  \sum_{u=1}^{\s{U}}\sum_{i=1}^{n} \frac{\s{Var}\left(G_{u,i}\right)}{\hat{f}_{u,i} ^{2}}\\
&=\sum_{u=1}^{\s{U}}\sum_{i=1}^{n} \frac{4m^3(P_{u,i}-Q_{u,i})^2(P_{u,i}+Q_{u,i})+2m^2(P_{u,i}+Q_{u,i})^{2}}{f_{u,i}^{2}} \\
& \stackrel{(a)}{\leq} 4 m^{3}\left(\sum_{u=1}^{\s{U}}\sum_{i=1}^{n} \frac{\left(P_{u,i}-Q_{u,i}\right)^{4}}{\hat{f}_{u,i}^{2}}\right)^{\frac{1}{2}}\left(\sum_{u=1}^{\s{U}}\sum_{i=1}^{n} \frac{\left(P_{u,i}+Q_{u,i}\right)^{2}}{\hat{f}_{u,i}^{2}}\right)^{\frac{1}{2}}\nonumber\\
&\hspace{3cm}+\sum_{u=1}^{\s{U}}\sum_{i=1}^{n} \frac{2 m^{2}\left(P_{u,i}+Q_{u,i}\right)^{2}}{\hat{f}_{u,i}^{2}} \\
& \stackrel{(b)}{\leq} 4 m^{3}\left(\sum_{u=1}^{\s{U}}\sum_{i=1}^{n} \frac{\left(P_{u,i}-Q_{u,i}\right)^{2}}{\hat{f}_{u,i}}\right)\left(\sum_{u=1}^{\s{U}}\sum_{i=1}^{n} \frac{\left(P_{u,i}+Q_{u,i}\right)^{2}}{\hat{f}_{u,i}^{2}}\right)^{\frac{1}{2}}\nonumber\\
&\hspace{3cm}+\sum_{u=1}^{\s{U}}\sum_{i=1}^{n} \frac{2 m^{2}\left(P_{u,i}+Q_{u,i}\right)^{2}}{\hat{f}_{u,i}^{2}} \\
&=4\left(m^{2} \sum_{u=1}^{\s{U}}\sum_{i=1}^{n} \frac{\left(P_{u,i}-Q_{u,i}\right)^{2}}{\hat{f}_{u,i} }\right)\left(m^{2} \sum_{u=1}^{\s{U}}\sum_{i=1}^{n} \frac{\left(P_{u,i}+Q_{u,i}\right)^{2}}{\hat{f}_{u,i}^{2}}\right)^{\frac{1}{2}}\nonumber\\
&\hspace{3cm}+\sum_{u=1}^{\s{U}}\sum_{i=1}^{n} \frac{2 m^{2}\left(P_{u,i}+Q_{u,i}\right)^{2}}{\hat{f}_{u,i}^{2}},
\end{align}
where (a) follows from the Cauchy-Schwartz inequality, and (b) is due to the monotonicity of the $\ell_{p}$ norm, i.e., for any vector $u$, $\|u\|_{2} \leq\|u\|_{1}$. Then,
\begin{align}
\s{Var}\left[\left.G \vphantom{\hat{f}_{u,i}}\right|\widehat{{f}_{u,i}}  \right] &=4\left(\mathbb{E}\left[\left.G \vphantom{\hat{f}_{u,i}}\right|\widehat{{f}_{u,i}}  \right]\right)\left(  \sum_{u=1}^{\s{U}}\sum_{i=1}^{n} \frac{m^{2}\left(P_{u,i}+Q_{u,i}\right)^{2}}{\hat{f}_{u,i}^{2}}\right)^{\frac{1}{2}}+\sum_{u=1}^{\s{U}}\sum_{i=1}^{n} \frac{2 m^{2}\left(P_{u,i}+Q_{u,i}\right)^{2}}{\hat{f}_{u,i}^{2}} \\
& \stackrel{(a)}{\leq} \frac{1}{40}\left(\mathbb{E}\left[\left.G \vphantom{\hat{f}_{u,i}}\right|\widehat{{f}_{u,i}}  \right]\right)^{2}+(160+2)   \sum_{u=1}^{\s{U}}\sum_{i=1}^{n} \frac{m^{2}\left(P_{u,i}+Q_{u,i}\right)^{2}}{\widehat{{f}_{u,i}}} \\
& \stackrel{(\mathrm{b})}{\leq} \frac{1}{40}\left(\mathbb{E}\left[\left.G \vphantom{\hat{f}_{u,i}}\right|\widehat{{f}_{u,i}}  \right]\right)^{2}+162   \sum_{u=1}^{\s{U}}\sum_{i=1}^{n} \frac{2 m^{2}\left(P_{u,i}-Q_{u,i}\right)^{2}}{\hat{f}_{u,i}^{2}}+162   \sum_{u=1}^{\s{U}}\sum_{i=1}^{n} \frac{4 m^{2} Q_{u,i}^{2}}{\hat{f}_{u,i}^{2}} \\
& \stackrel{(\mathrm{c})}{\leq} \frac{1}{40}\left(\mathbb{E}\left[\left.G \vphantom{\hat{f}_{u,i}}\right|\widehat{{f}_{u,i}}  \right]\right)^{2}+324   \sum_{u=1}^{\s{U}}\sum_{i=1}^{n} \frac{m^{2}\left(P_{u,i}-Q_{u,i}\right)^{2}}{\hat{f}_{u,i}}+648   \sum_{u=1}^{\s{U}}\sum_{i=1}^{n} m^{2} Q_{u,i}^{2} \\
&=\frac{1}{40}\left(\mathbb{E}\left[\left.G \vphantom{\hat{f}_{u,i}}\right|\widehat{{f}_{u,i}}  \right]\right)^{2}+324 \mathbb{E}\left[\left.G \vphantom{\hat{f}_{u,i}}\right|\widehat{{f}_{u,i}}  \right]+648 m^{2}\sum_{u=1}^{\s{U}}\|Q_{u}\|_{2}^{2}\\&=\frac{1}{40}\left(\mathbb{E}\left[\left.G \vphantom{\hat{f}_{u,i}}\right|\hat{f}_{u,i},u\in[\s{U}],i\in[n]\right]\right)^{2}+324 \mathbb{E}\left[\left.G \vphantom{\hat{f}_{u,i}}\right|\hat{f}_{u,i},u\in[\s{U}],i\in[n]\right]\nonumber\\
&\hspace{1cm}+648 m^{2}\sum_{u=1}^{\s{U}}\|Q_{u}\|_{2}^{2},
\end{align}
where in (a) we use the fact that $2 a b \leq a^{2}+b^{2}$, (b) follows from $(a+b)^{2} \leq 2(a-b)^{2}+4 b^{2}$, and finally (c) is because $\hat{f}_{u,i} \geq 1$.    
\end{proof}

\subsubsection{Proof of Theorem \ref{theorem:iid_sample_complexity}.}

We start with the case where $m \geq n$. By Chebyshev's inequality,
\begin{align}
    \pr\left(\left.|G- \mu | \leq \frac{\sigma}{\sqrt{\delta}} \vphantom{\hat{f}_{u,i}}\right|\hat{f}_{u,i},u\in[\s{U}],i\in[n]\right) \geq 1- \delta,
\end{align} 
where 
\begin{align}
\mu&=\mathbb{E}\left[\left.G \vphantom{\hat{f}_{u,i}}\right|\widehat{{f}_{u,i}} \text { for } u\in[\s{U}], i \in[n]\right],\\
\sigma^2&=\s{Var}\left[\left.G \vphantom{\hat{f}_{u,i}}\right|\widehat{{f}_{u,i}} \text { for } u\in[\s{U}],i \in[n]\right].
\end{align}
Accordingly, using Corollary~\ref{cor:1}, we get that with probability at least $1-\delta$,
\begin{align}
 \frac{\delta_1}{c_1}\min \left(\frac{m^{3 / 2} \sum_{u=1}^{\s{U}}\left\|P_{u} - Q_{u}\right\|_{1} }{2 \s{U}\sqrt{n}}, \frac{m^{2}\left(\sum_{u=1}^{\s{U}}\left\|P_{u} - Q_{u}\right\|_{1}\right)^{2}}{6\s{U}n}\right) - m\sqrt{\frac{10c_{3}\s{U}}{\delta\delta_1 n}} \leq G, \label{lower bound of G first}
\end{align}
and
\begin{align}
G \leq\frac{c_{2}}{\delta_1}\sum_{u=1}^{\s{U}} \frac{m^{3 / 2}\|P_{u}-Q_{u}\|_{1}}{n^{\frac{1}{2}}} + m\sqrt{\frac{10c_{3}\s{U}}{\delta\delta_1 n}}. \label{upper bound of G 1}
\end{align}
We will consider the two possible cases where either $\sum_{u=1}^{\s{U}}\|P_{u}-Q_{u}\|_{1}\geq \varepsilon_{2}\s{U}$ or $\sum_{u=1}^{\s{U}}\|P_{u}-Q_{u}\|_{1}\leq \varepsilon_{1}\s{U}$. Starting with the former, the lower bound in (\ref{lower bound of G first}) reduces to
\begin{align}
G &\geq   \frac{\delta_1}{c_1}\min \left(\frac{m^{3 / 2} \varepsilon_{2}\s{U} }{2 \s{U}\sqrt{n}}, \frac{m^{2}\varepsilon_{2}^{2}\s{U}^{2}}{6\s{U}n}\right) - m\sqrt{\frac{10c_{3}\s{U}}{\delta\delta_1 n}} \\
&\geq \frac{\delta_1}{12c_{1}} \min \left(\frac{m^{3 / 2} \varepsilon_{2}}{\sqrt{n}}, \frac{m^{2} \varepsilon_{2}^{2}\s{U}}{n}\right),
\end{align}
where the second inequality holds as long as $m \geq \max
\left\{\frac{12}{5}^2\frac{10c_3c_1^2}{\delta_1^2}\frac{\s{U}}{\delta\varepsilon_2^2},\frac{12c_1}{\delta_1}\sqrt{\frac{10c_3}{\delta_1}}\sqrt{\frac{n}{\s{U}\delta\varepsilon_{2}^{4}}}\right\} = \s{C}\max \left\{\frac{\s{U}}{\delta\varepsilon_2^2},\sqrt{\frac{n}{\s{U}\delta\varepsilon_{2}^{4}}}\right\}$, for some constant $\s{C}>0$. Therefore, with probability at least $1-\delta$ the test announces that $\sum_{u=1}^{\s{U}}\|P_{u}-Q_{u}\|_{1} \geq \s{U}\varepsilon_{2}$ correctly.

On the other hand, in the case where $\sum_{u=1}^{\s{U}}\|P_{u}-Q_{u}\|_{1} \leq \varepsilon_{1}\s{U}$, using Corollary~\ref{cor:1} once again, we get that with probability at least $1-\delta$,
\begin{align}
G &\leq \frac{c_{2}}{\delta_1}\frac{m^{3 / 2}\varepsilon_{1}\s{U}}{\sqrt{n}}+m\sqrt{\frac{10c_{3}\s{U}}{\delta\delta_1 n}} \\
&\leq \frac{\delta_1}{24c_{1}} \min \left(\frac{m^{3 / 2} \varepsilon_{2}}{\sqrt{n}}, \frac{m^{2} \varepsilon_{2}^{2}\s{U}}{n}\right),
\end{align}
where the second inequality holds as long as $m>\max\left(\frac{\delta_1 \varepsilon_2}{c_{2}\varepsilon_1} \sqrt{\frac{1}{\s{U}\delta}},\frac{c_2^2n\varepsilon_1^2}{\delta_1^2\varepsilon_2^4}\right)=\s{\tilde{C}}\max\left(\frac{ \varepsilon_2}{\varepsilon_1} \sqrt{\frac{1}{\s{U}\delta}},\frac{n\varepsilon_1^2}{\varepsilon_2^4}\right)$, for some constant $\tilde{\s{C}}>0$. Therefore, with probability at least $1-\delta$ the test announces that
$\sum_{u=1}^{\s{U}}\|P_{u}-Q_{u}\|_{1}\leq \varepsilon_{1}$ correctly. 

Finally, we turn to the case $m \leq n$. Using Corollary~\ref{cor:2} and repeating the same arguments above, we get that the test announces $\sum_{u=1}^{\s{U}}\|P_{u}-Q_{u}\|_{1}\geq \varepsilon_{2}\s{U}$ correctly if $m\geq n\frac{\varepsilon_1}{\varepsilon_2^2}$, while announces $\sum_{u=1}^{\s{U}}\|P_{u}-Q_{u}\|_{1}\leq \varepsilon_{1}\s{U}$ correctly, as long as $m\geq\frac{n^{2/3}}{\s{U}\varepsilon_2^{4/3}}$. This concludes the proof.

\subsection{Proof of Theorem~\ref{lemma:m-Jointly-k-cover time and tolerant-closeness-testing Markov chains}}



We start with the following lemma generalizes the ``exponential decay lemma" in \cite{pmlr-v132-chan21a} for the case where we have $m$ independent Markov chains. This result will be essential in the proof of Theorem~\ref{lemma:m-Jointly-k-cover time and tolerant-closeness-testing Markov chains}. For simplicity of notations, we will sometimes suppress the dependency of the covering quantities on $\mathscr{M}$. 

\begin{lemma}[Exponential decay]\label{lemma:m-Joint-k-Cover Exponential Decay} For $\s{M}$ independent irreducible Markov chains $\{\s{Z}_{m,1}^\infty\}_{m=1}^{\s{M}}$, on the same state space $[n]$, for any $k,\s{L}\in\mathbb{N}$, and any initial distribution $\mathbf{q}$ over $[n]^\s{M}$, 
we have 
\begin{align}
    \mathbb{P}\left(\tau_{\s{cov}}^{(k)}(m) \geq e \s{L} t_{\s{cov}}^{(k)}(m;\mathscr{M})\right) \leq e^{-\s{L}}.
\end{align}
\end{lemma}

\begin{proof}[Proof of \text{Lemma} \ref{lemma:m-Joint-k-Cover Exponential Decay}]
Consider $\tau_{\s{cov}}^{(k)}(\s{M};\mathscr{M})$ with any fixed starting states $\s{Z}_{\ell,1}=v_\ell$, for $\ell\in[\s{M}]$. By Markov's inequality, 
\begin{align}
\label{eq:Markov's inequality}
\mathbb{P}(\tau_{\s{cov}}^{(k)}(\s{M}) \geq e t_{\s{cov}}^{(k)}(\s{M};\mathscr{M})) &\leq \mathbb{P}\left(\tau_{\s{cov}}^{(k)}(\s{M}) \geq e \mathbb{E}\left[\left.\tau_{\s{cov}}^{(k)}(\s{M}) \right| \s{Z}_{\ell,1}=v_\ell,\;\forall j\in[\s{M}]\right]\right) \\
&\leq e^{-1}.
\end{align}
Note that this inequality holds for any initial states $\mathbf{v}=\left(v_1,\dots,v_{\s{M}}\right)\sim\mu$, where $\mu$ is any discrete distribution over $[n]^\s{M}$. We next analyze sub-trajectories of our Markov chain of length $\nu\triangleq e t_{\s{cov}}^{(k)}(\s{M};\mathscr{M})$. Specifically, for any $1\leq\ell\leq \s{L}$, we define $\calE_\ell$ as the event that the set of $\s{M}$ sub-trajectories of the Markov chains $\{\s{Z}_{m,(\ell-1)\nu+1}^{\ell\nu}\}_{m=1}^{\s{M}}$ jointly cover the state space $k$ times. According to \eqref{eq:Markov's inequality}, we have 
\begin{align}
\mathbb{P}\left(\calE_1^{c}\right)&=\mathbb{P}\left(\tau_{\s{cov}}^{(k)}(\s{M}) \geq e t_{\s{cov}}^{(k)}(\s{M};\mathscr{M})\right)\leq e^{-1}.
\end{align}
Denote the distribution of $\left\{\s{Z}_{m,\nu}\right\}_{m\in[\s{M}]}$ conditioned on $\calE_1^{c}$ by $\mu^{\prime}$, and let $\tau_{\s{cov}}^{(k)^{\prime}}(\s{M})$ be the $\s{M}$-joint $k$-cover time of $\{\s{Z}_{m,\nu+1}^{\infty}\}_{m=1}^{\s{M}}$. We have,
\begin{align}
\mathbb{P}\left(\calE_2^{c} \vert \calE_1^{c}\right)&=\mathbb{P}\left(\left.\tau_{\s{cov}}^{(k)^{\prime}}(\s{M}) \geq e t_{\s{cov}}^{(k)}(\s{M}) \right| \tau_{\s{cov}}^{(k)}(\s{M}) \geq\right.\left.e t_{\s{cov}}^{(k)}(\s{M};\mathscr{M})\right)\\
&=\mathbb{P}\left(\left.\tau_{\s{cov}}^{(k)^{\prime}}(\s{M}) \geq e t_{\s{cov}}^{(k)}(\s{M};\mathscr{M}) \right| \left\{\s{Z}_{m,\nu}\right\}_{m\in[\s{M}]} \sim \mu^{\prime}\right) \leq e^{-1},
\end{align} 
where we used the fact that $\calE_1$ is determined by $\{\s{Z}_{m,1}^{\nu}\}_{m\in[\s{M}]}$, and that by the Markov property the event $\calE_2$ do not depend on $\{\s{Z}_{m,1}^{\nu-1}\}_{m\in[\s{M}]}$. Thus, 
\begin{align}
\mathbb{P}\left(\calE_1^{c} \cap \calE_2^{c}\right)=\mathbb{P}\left(\calE_1^{c}\right) \mathbb{P}\left(\calE_2^{c} \mid \calE_1^{c}\right) \leq e^{-2}.
\end{align}
Using the same arguments above, by induction, we can show that
$\mathbb{P}\left(\cap_{i \in[\s{L}]} \calE_{i}^{c}\right) \leq e^{-\s{L}}$. Define $\calE$ as the event that the set of $\s{M}$ sub-trajectories of the Markov chains $\{\s{Z}_{m,1}^{\s{L}\nu}\}_{m=1}^{\s{M}}$ jointly cover the state space $k$ times. Then, it is clear that $\cup_{i \in[\s{L}]} \calE_{i}\subseteq\calE$, and thus, $\mathbb{P}\left(\calE^{c}\right) \leq \mathbb{P}\left(\cap_{i \in[\s{L}]} \calE_{i}^{c}\right) \leq e^{-\s{L}}$, which concludes the proof.
\end{proof}
We are now in a position to prove Theorem \ref{lemma:m-Jointly-k-cover time and tolerant-closeness-testing Markov chains}. Recall that $\bar{m}=m(n,\varepsilon_1,\varepsilon_2,\delta/4n)$ is the sample complexity guarantee associated with Algorithm~\ref{algo:tester_HT_4}. For $u\in[\s{U}]$ and $\ell\in\mathbb{N}$ we define the events,
\begin{align}
\calE_{\s{R},u}(\ell)&\triangleq\left\{\sum_{m=1}^{\s{M}}\calN_{i}^{\f{x}_{m,u}^{\s{R}}}(\ell)\geq \bar{m},\;\forall i\in[n]\right\},\label{eqn:ERu}\\
\calE_{\s{F},u}(\ell)&\triangleq\left\{\sum_{m=1}^{\s{M}}\calN_{i}^{\f{x}_{m,u}^{\s{F}}}(\ell)\geq \bar{m},\;\forall i\in[n]\right\}.\label{eqn:EFu}
\end{align}
Furthermore, for any $i\in[n]$, we define $\tilde\calE_i$ as the event that steps 6-8 in Algorithm~\ref{algo:TolerantClosenessTestingRegu} return ``\texttt{NO}", namely, that the first $\bar{m}$ succeeding samples of state $i\in[n]$ in the union of the $\s{M}$ Markov trajectories do not pass the i.i.d. tester in Algorithm~\ref{algo:tester_HT_4}. 

To establish Theorem~\ref{lemma:m-Jointly-k-cover time and tolerant-closeness-testing Markov chains}, we consider the two possible cases where $\mathbb{V}_{\s{filter}}\leq \varepsilon_1$, and the complementary case where $\mathbb{V}_{\s{filter}}\geq \varepsilon_2$. Starting with the former, according to Lemma~\ref{lemma:m-Joint-k-Cover Exponential Decay}, by taking $\s{L}=\log \frac{4\s{U}}{\delta}$, we get $\mathbb{P}\left(\tau_{\s{cov}}^{(\bar{m})}(\s{M};\s{Q}_u^{\s{F}}) \geq e t_{\s{cov}}^{(\bar{m})}(\s{M};\s{Q}_u^{\s{F}}) \log \frac{4\s{U}}{\delta}\right) \leq \frac{\delta}{4\s{U}}$ and $\mathbb{P}\left(\tau_{\s{cov}}^{(\bar{m})}(\s{M};\s{P}_u^{\s{R}}) \geq e t_{\s{cov}}^{(\bar{m})}(\s{M};\s{P}_u^{\s{R}}) \log \frac{4\s{U}}{\delta}\right) \leq \frac{\delta}{4\s{U}}$, for all $u\in[\s{U}]$. Thus, for a length $\ell=\s{T}$ trajectory in \eqref{eqn:ERu}--\eqref{eqn:EFu}, where $\s{T}$ is as stated in Theorem~\ref{lemma:m-Jointly-k-cover time and tolerant-closeness-testing Markov chains}, i.e.,
\begin{align}
    \s{T}=e \log \frac{4\s{U}}{\delta} \max_{u\in[\s{U}]}\max_{\s{W}\in\{\s{Q}_u^{\s{F}},\s{P}_u^{\s{R}}\}}t_{\s{cov}}^{\bar{m}}\left(\s{M};\s{W}\right),
\end{align} 
we will have $\bar{m}$ samples for each state in $[n]$ with probability $\mathbb{P}\left(\calE_{\s{R},u}(\s{T})\right) \geq 1- \frac{\delta}{4\s{U}}$, for any $u\in[\s{U}]$. Similarly, $\mathbb{P}\left(\calE_{\s{F},u}(\s{T})\right) \geq 1-\frac{\delta}{4\s{U}}$, for any $u\in[\s{U}]$. Thus, by a union bound over the two chains and the set all users $u\in[\s{U}]$, the probability of passing the condition in step 5 of Algorithm \ref{algo:TolerantClosenessTestingRegu} is at least $\geq 1-\frac{\delta}{2}$. Furthermore, by the sample complexity guarantee in Theorem~\ref{theorem:iid_sample_complexity} associated with the i.i.d. tester in Algorithm~\ref{algo:tester_HT_4}, we have $\mathbb{P}(\tilde{\calE}_{i}) \leq \frac{\delta}{ 4n}$, implying that, $\mathbb{P}(\cup_{u\in\s{U}} \calE_{\s{R},u}^c(\s{T}) \cup \calE_{\s{F},u}^c(\s{T}) \cup \tilde{\calE}_1 \ldots \cup \tilde{\calE}_{n}) \leq \frac{3\delta}{4}$. Thus, with probability at least $1-\delta$, the tester will return \texttt{YES}.

Next, we move forward to the complementary case. Note that the only case the algorithm outputs \texttt{YES} is when it do not pass steps 4 and 7 in Algorithm \ref{algo:TolerantClosenessTestingRegu} for all states, which means it will have enough samples for testing each state, and the i.i.d. tester outputs \texttt{YES} for all sub-tests. Since $\mathbb{V}_{\s{filter}}\geq \varepsilon_2$ implies that there exists $i_\star \in[n]$ such that $\sum_{u=1}^{\s{U}} \left\|\mathbf{P}^{\f{R}}_{u}(i_\star)- \mathbf{Q}^{\f{F}}_{u}(i_\star) \right\|_{1}\geq\s{U}\varepsilon_2$, this guarantees that the sub-test for $i_\star$ will return \texttt{NO} with probability $\mathbb{P}(\tilde{\calE}_{i_\star}) \geq 1-\frac{\delta}{4n}$, again due to the sample complexity guarantees for the i.i.d. tester in Theorem~\ref{theorem:iid_sample_complexity}. Thus, the probability for the whole procedure outputting \texttt{YES} is $\mathbb{P}(\cap_{u\in[\s{U}]}\calE_{\s{R},u}(\s{T}) \cap \calE_{\s{F},u}(\s{T}) \cap \tilde{\calE}_1^{c} \ldots \cap \tilde{\calE}_{n}^{c}) \leq \mathbb{P}(\tilde{\calE}_{i_\star}^{c}) \leq \delta / 4 n$. 

Combining both cases above, it is clear that Algorithm~\ref{algo:TolerantClosenessTestingRegu} will output the correct answer with probability at least $1-\delta$.

\subsection{Proof of Lemma~\ref{$m$-joint-$k$-cover time upper bound}}

To prove Lemma~\ref{$m$-joint-$k$-cover time upper bound} we start with the following concentration bound on the random $\ell$-joint $k$-hitting time $\tau_{\s{hit}}^{(k)}(\ell;i)$, defined as the first time when a particular state $i\in[n]$ is visited $k$ times jointly by the $\ell$ Markov chains. Mathematically, we define 
\begin{align}
    \tau_{\s{hit}}^{(k)}(\ell;i)\triangleq\inf\ppp{t\geq0:\sum_{j=1}^{\ell}\calN_{i}^{\s{Z}_{j}}(t)\geq k},
\end{align}
for $i\in[n]$. The $\ell$-joint $k$-hitting time is then defined as
\begin{align}
    t_{\s{hit}}^{(k)}(\ell)\triangleq\max _{i_\in[n],\mathbf{v} \in[n]^\ell} \mathbb{E}\left[\left.\tau_{\s{hit}}^{(k)}(\ell;i) \right| \s{Z}_{1,1}=v_1,\s{Z}_{2,1}=v_2,...,\s{Z}_{\ell,1}=v_\ell\right].  
\end{align}
Finally, we define the $\ell$-joint return time. For some state $i\in[n]$, the random $\ell$-joint return time $\tau_{\s{ret}}(\ell;i)$ is the first time one of the $\ell$ Markov chains starting at $i$ return to $i$. The $\ell$-joint return time is then defined as $t_{\s{ret}}(\ell;i) = \bE[\tau_{\s{ret}}(\ell;i)\vert \s{Z}_{1,1}=i,\s{Z}_{2,1}=i,...,\s{Z}_{\ell,1}=i]$. It is standard result that for irreducible chains $t_{\s{ret}}(1;i) =1/\pi_i$, where $\pi$ is the stationary distribution.
\begin{lemma}\label{lem:concHit}
Let $\s{Z}_{1,1}^{\infty},\s{Z}_{2,1}^{\infty},...,\s{Z}_{\ell,1}^{\infty}$ be $\ell$-independent infinite trajectories drawn by the same Markov chain $\mathscr{M}$. Then, for any $i\in[n]$,
\begin{align}
    \pr\p{\tau_{\s{hit}}^{(k)}(\ell;i)\geq t}\leq \exp\p{-\frac{t}{e\gamma_i}},
\end{align}
for any $t\geq0$, where $\gamma_i\triangleq t_{\s{hit}}^{(1)}(\ell)+\frac{k}{\pi_i}$.
\end{lemma}

\begin{proof}[Proof of Lemma~\ref{lem:concHit}]
    First, we note that by the Markov property we have 
    \begin{align}
        \tau_{\s{hit}}^{(k)}(\ell;i) &= \tau_{\s{hit}}^{(1)}(\ell;i)+(k-1)\cdot t_{\s{ret}}(\ell;i)\\
        &\leq\tau_{\s{hit}}^{(1)}(\ell;i)+\frac{k}{\pi_i}\\
        &\leq t_{\s{hit}}^{(1)}(\ell)+\frac{k}{\pi_i} = \gamma_i,
    \end{align}
    where the last inequality holds by definition with probability one. Thus, by Markov inequality we get,
    \begin{align}
        \pr\pp{\tau_{\s{hit}}^{(k)}(\ell;i)\geq e\gamma_i}\leq \frac{1}{e}.
    \end{align}
    Using the same arguments as in the proof of the exponential decay result in Lemma~\ref{lemma:m-Joint-k-Cover Exponential Decay}, we can show that for any $\ell\geq1$,
    \begin{align}
        \pr\pp{\tau_{\s{hit}}^{(k)}(\ell;i)\geq e\kappa\gamma_i}\leq e^{-\kappa}.
    \end{align}
    Thus, the result follows by taking $\kappa = t/e\gamma_i$. 
\end{proof}
    Using the above result we are now in a position to prove Lemma~\ref{$m$-joint-$k$-cover time upper bound}. First, it is clear that $\tau_{\s{cov}}^{(k)}(\ell)\leq \max_{i\in[n]}\tau_{\s{hit}}^{(k)}(\ell;i)$. Thus, by the union bound and Lemma~\ref{lem:concHit}, we have for any $t\geq0$,
    \begin{align}
        \pr\pp{\tau_{\s{cov}}^{(k)}(\ell)\geq t}\leq \sum_{i\in[n]}\exp\p{-\frac{t}{e\gamma_i}}\leq n\exp\p{-\frac{t}{e\min_{i\in[n]}\gamma_i}}.
    \end{align}
    Therefore, we have,
    \begin{align}
        \bE[\tau_{\s{cov}}^{(k)}(\ell)] &= \int_0^\infty\pr\p{\tau_{\s{cov}}^{(k)}(\ell)\geq t}\mathrm{d}t\\
        &\leq \int_0^{e\min_{i\in[n]}\gamma_i\log n}\mathrm{d}t+n\int_{e\min_{i\in[n]}\gamma_i\log n}^\infty\exp\p{-\frac{t}{e\min_{i\in[n]}\gamma_i}}\mathrm{d}t\\
        & =e\min_{i\in[n]}\gamma_i\log n+e^2\min_{i\in[n]}\gamma_i.
    \end{align}
Thus, it follows that $t_{\s{cov}}^{(k)}(\ell) = \calO\p{\min_{i\in[n]}\gamma_i\log n} =\calO\p{t_{\s{hit}}^{(1)}(\ell)\log n+\frac{k\log n}{\pi_\star}}$, and therefore, $t_{\s{cov}}^{(k)}(\ell) = \calO\p{t_{\s{cov}}^{(1)}(\ell)\log n+\frac{k\log n}{\pi_\star}}$. Finally, combining Theorem 3.2, Lemma 4.3, and  Theorem 4.7 in \cite{rivera_sauerwald_sylvester_2023}, we have that,
\begin{align}
   t_{\s{cov}}^{(1)}(\ell)= \calO\p{\max\ppp{t_{\s{mix}},\frac{t_{\s{cov}}\log n}{\ell}}},
\end{align}
which concludes the proof.

\section{Conclusion and Future Research}

In this paper, we modeled the relationship between the three stakeholders: the platform, the users, and the auditor. The essence of the modeling is that from the auditor's perspective the platform is a content-generating system formulated by a multidimensional first order Markov chain (as the fixed number of pieces of the content appearing on each feed), where at every time step the platform samples a new feed, according to the Markov transition-matrix (conditional probability). We developed an auditing method that tests whether there are unexpected deviations in the user's decision-making process over a predefined time horizon. Unexpected deviations in the user's decision-making process might be a result of the selective filtering of the contents to be shown on the user's feed in comparison to what would be the users' decision-making process under natural filtering. We proposed also an auditing procedure for online counterfactual regulations.

There are several exciting directions for future work, including the following. A major goal going forward is to evince our auditing procedure on real social media content. Specifically, while our work propose a theoretical framework for SMP auditing, we left several fundamental questions that revolve around implementability, such as, how do we know if the framework is effective or useful? What are the metrics that should be used? We are currently investigating these kind of questions. From the technical perspective, there are many interesting generalizations an open questions that we plan to investigate. For example, studying a sequential version of the testing problems proposed in this paper are of particular importance. Indeed, in real-world platforms decisions must be taken as quickly as possible so that proper countermeasures can be taken to suppress regulation violation. Moreover, real-world networks are gigantic and therefore it is quite important to study the performance of low-complexity algorithms. Also, it is of both theoretical and practical importance to consider more general probabilistic models which will, for example, capture the dynamic relationships between users, the varying influence of individual users within the platform, and in general weaken some of the assumptions we made about the behaviour of users, the platform/algorithm, the social relationships and dynamics. It would be interesting to consider more complicated/real-world motivated generative models, such as, higher-order Markov chains, as well as simple structured dependencies among the $\s{M}$ contents. In this paper we considered testing against a single agree upon definition for a reference feed. However, there is more than one ``natural/fair" way to filter contents. Accordingly, it would be more robust and general to require closeness to the set of ``natural/fair" references, which may be very different from one another. From the conceptual perspective, while our paper propose several definition for the notions of ``variability" and ``violation", there probably are other possible definitions, which take into account some perspectives of responsible regulation which we ignored, and are important to investigate. 

Finally, it should be clear that each approach, worst-case \cite{cen2020regulating} or average, has its own advantages and disadvantages. For example, the worst-case approach might be sensitive to adversarial users; in real-world SMPs, where any party is free to create a user without any supervision, a set of adversarial users can act as more naive/gullible than the most gullible user already in existence, and thus fool the auditor. Also, the worst-case approach prevents all users from gradually changing their opinions. This is because, under this approach, the auditing process will immediately result in a violation when the most gullible user alters its opinion slightly. As a result, all other users will not have the opportunity to make slow and natural changes to their opinions, as they would with our average approach. In some sense, the above problematic issues are less severe/relevant in our average approach. In the average approach, the auditing procedure would not prevent the SMP from identifying a set $O(1)$-many users who, for example, are most likely to tip the outcome of an election and promoting one presidential candidate to them, while using the reference feed itself for the remaining users. However, if the set of chosen users is ``large enough", coupled with good-faith effort to choose and test features, this issue can be resolved. It is clear, nonetheless, that further research of both approaches (and perhaps others) is needed so as to achieve better understanding of this complicated problem of auditing/regulating SMPs.

\acks{This work was supported by the ISRAEL SCIENCE FOUNDATION (grant No. 1734/21).}

\appendix

\section{Detailed Construction of Our Framework}

In this appendix, we provide more detailed discussions about the framework, definitions, and assumptions presented in Section~\ref{section-Framework}.

\subsection{User-platform relationship}\label{subsec:user-platform}

\paragraph{Users.} We describe the \emph{users learning and decision-making pipeline}. As users browse through their feeds, they implicitly form internal \emph{beliefs} about the observed contents, and based on those beliefs they later take \emph{actions/decisions}. For example, how individuals vote or the products they buy are decisions that are affected by the content they see on social media. In addition, the decisions does not have to occur on the platform. For instance, the platform could show information on COVID-19, but the decision could be whether to get the vaccine. Let us formulate this mathematically. Let $\Omega$ be a compact metrizable set of possible states; this set can be finite, countably infinite, or continuums, and its elements $\omega\in\Omega$ can be either scalars or vectors. At each time step $t\geq1$, each user $u\in[\s{U}]$ is associated with a belief $\calB^{\s{F}}_{u,t}\in\Delta(\Omega)$, where $\Delta(\Omega)$ is the simplex of probability distributions over the state space. At start $t=0$, without loss of essential generality, we may assume that $\calB^{\s{F}}_{u,0}=\s{Uniform}(\Omega)$, for all $u\in[\s{U}]$. The belief $\calB^{\s{F}}_{u,t}$ is a posterior distribution on $\Omega$ \emph{conditioned} on the information available to user $u$ at time $t$. This information consists of the observed feeds $\{\bX^{\s{F}}_{u}(\ell)\}_{\ell\leq t}$. The total history information available to user $u$ at time $t$ by  $\s{H}^{\s{F}}_{t,u}\triangleq\ppp{\bX^{\s{F}}_{u}(\ell):\ell\leq t}$. Accordingly, user $u$'s belief at time $t$ is defined as $\calB^{\s{F}}_{u,t}(\mathrm{d}\omega,\s{h}_{t,u}) \triangleq \pr^{\s{F}}(\mathrm{d}\omega\vert\s{H}^{\s{F}}_{t,u}=\s{h}_{t,u})$, for a given sequence of feeds $\s{h}_{t,u}$. Based on the beliefs users take decisions (or, actions); each user have a set of possible \emph{actions} at time $t\geq0$. For user $u\in\s{U}$, let $\calA_u(t)$ denote a compact metrizable action space, and $\s{A}_{u,i}(t)\in\calA_u(t)$ be the $i$th action. Also, let $\calU_u:\Omega\times\calA_u\to\mathbb{R}$ be (possibly continuous) user $u$'s utility function. Consequently, for any belief $\calB^{\s{F}}_{u,t}\in\Delta(\Omega)$ and a utility function $\calU_u$ we define $\s{br}^{\s{F}}_{u,t}(\s{h}_{t,u})$ as the set of actions that maximizes user $u$'s expected utility, i.e.,
\begin{align*}
    \s{br}^{\s{F}}_{u,t}(\s{h}_{t,u})\triangleq\ppp{a\in\calA_u:a\in\argmax_{b\in\calA_u}\int_{\Omega}\calU_u(\omega,b)\calB^{\s{F}}_{u,t}(\mathrm{d}\omega,\s{h}_{t,u})}.   
\end{align*}

\subsection{Feeds construction and auditor-platform interaction}\label{subsec:RegplatRelation}
We now switch our focus to formalize the setup for the auditor-platform interaction.

\paragraph{Platform filtering.} An important component of our model is related to the question of \emph{how feeds are filtered}? As mentioned before, feeds are chosen by the platform using a black-box filtering algorithm, which is utilized to maximize a certain reward function. The filtering algorithm is fed with an extensive amount of inputs that the platform uses to filter, such as, current available contents, past feeds, users interaction history, users feedback (e.g., users ``sentiments" which are certain complex functions of the users beliefs), the users social network topology, and so on. The reward function reflects the platform's objective. For example, it may balance factors like advertising revenue, personalization, user engagement (e.g., the predicted number of clicks), content novelty, acquisition of new information about users, cost of operations, or a combination of these and other factors. We denote the platform's reward function by $\s{Rew}^{\s{F}}_t:\calX^\s{M}\times\calP_{t-1}\to\mathbb{R}$, where $\calP_{t-1}$ captures the inputs mentioned above, and accordingly, 
\begin{align}
\bX^{\s{F}}_{u}(t) = \argmax_{\f{x}\in\calX^\s{M}}\s{Rew}^{\s{F}}_t(\f{X},\calP_{u,t-1}),\label{eqn:rewarddef}
\end{align}
where, again, $\calP_{u,t-1}$ captures the platform external data used for filtering. For now, we leave both $\s{Rew}^{\s{F}}_t$ and $\calP_{t-1}$ unspecified.

\paragraph{Filtered vs. reference feeds.} The discussion in the background section about the regulation boundary and motivation suggests a neat and consistent formulation for the auditor's objective. Following \cite{wachter2019right,ghosh2019new,cen2023userdriven,petty2000marketing}, we define a reference (or, competitive) boundary that is formed based on the users consent, and its location is determined by domain experts. While user $u$'s filtered feed $\bX_{u}^{\s{F}}(t)$ at time $t$ is chosen by the platform in a certain reward-maximizing methodology,
On the other hand, the \emph{reference feeds} $\bX_u^{\s{R}}(t)$ could have been hypothetically selected by the platform if it strictly followed the consumer-provider agreement. These reference feeds are specific to each user $u\in[\s{U}]$ and time $t$. In this scenario, the platform would construct the feed based solely on the user's interests, without any subjective preferences influencing the content selection. Essentially, the only natural situation where the platform can filter content without introducing any subjective bias into the user's decision-making process and actions is by selecting feasible content that maximizes the user's benefit/reward. This approach ensures that the user's feed reflects their own preferences, which may align with the platform's benefits at times, but not necessarily always. Mathematically, the user's exclusive benefit is quantified by a personal reward function that encompasses only the components measuring the user's benefits. We formulate this objective rigorously, while elucidates the difference between the filtered and reference feeds.
\begin{definition}[Construction of reference feeds] Suppose that the platform's reward objective function can be written as following type
\begin{align*}
    \s{Rew}^{\s{F}}_t(\bX,\s{P}_{i,t-1}) \triangleq \s{Rew}_{t,\s{per}}(\bX,\s{P}_{i,t-1})+\s{Rew}_{t,\s{rev}}(\bX,\s{P}_{i,t-1})+\s{Rew}_{t,\s{self}}(\bX,\s{P}_{i,t-1}),
\end{align*}
where $\s{Rew}_{t,\s{per}}$ is the reward gained by those feeds which are personalized to the user, $\s{Rew}_{t,\s{rev}}$ is the revenue-related reward gained by advertisements, and $\s{Rew}_{t,\s{self}}$ predicts the reward associated with the information the platform would gain from platform ``selfish" aspects (e.g., running a social experiment on the user).
Without the loss of generality, assume that the first two types of rewards are consistent with the consumer-provider agreement, but the last one is not. Then, the reference feed could be the one that maximize the contribution of the first two types of rewards, namely, $$\s{Rew}^{\s{R}}_t(\bX,\s{P}_{i,t-1}) \triangleq \s{Rew}_{t,\s{per}}(\bX,\s{P}_{i,t-1})+\s{Rew}_{t,\s{rev}}(\bX,\s{P}_{i,t-1}).$$ 
\end{definition}

It should be emphasized here that the specific reference feed construction we described above is just one possible example; our results and algorithms only require that there is some fixed reference feed (per user).

 \paragraph{Auditor's generative modeling.}
The AF mechanism is not known and should not be disclosed to the auditor. Nonetheless, it should be clear that for the auditor to be able to inspect the SMP, something about the feeds generation process must be assumed. In this paper, we assume that from the auditor's point of view, the feeds are generated at random, and we denote the conditional law of the feed at time $t$ conditioned on history feeds $\s{h}_{t-1,u}$ by $\pr^{\s{F}}_{u,t}(\s{h}_{t-1,u}) \triangleq \pr(\bX^{\s{F}}_{u}(t)\vert\s{H}_{t-1,u}^{\s{F}}=\s{h}_{t-1,u})$, for user $u$. Later on, for the framework to be mathematical tractable, we will place additional assumptions on the family of distributions.

\paragraph{Time dependent counterfactual regulations.} Above, we have focused on the ``filtered vs. reference" feeds approach. We propose the following as an alternative. Let $\calS$ be a \emph{regulatory statement} that an inspector (or, perhaps, the platform itself) wish to test. For example, $\calS$ could be: ``\emph{The platform should produce similar feeds, in the course of a given time horizon $\s{T}$, for users who are identical except for property $\mathscr{P}$}", where $\mathscr{P}$ could be ethnicity, sexual orientation, gender, a combination of these factors, etc. Let $\calU_{\mathscr{P}}\subset\s{U}\times\s{U}$ be a subset of pairs of users that comply with $\mathscr{P}$. Then, for any pair of users $(i,j)\in\calU_{\mathscr{P}}$, the inspector's objective is to determine whether the platform's filtering algorithm cause user $i$'s and user $j$'s beliefs and actions to be significantly different. We formulate this objective rigorously in the next section. We mention here that a similar approach to the above was proposed recently in \cite{cen2020regulating}, assuming a time-independent static model. Our study first focuses on constructing a regulation procedure given the first usable form, filtered vs. reference feeds. However, we will later reveal that a regulation procedure for the second form, counterfactual regulations, could be constructed using two parallel procedures of the first form.

\paragraph{Hypothesis testing.} The auditor's goal is to determine whether the platform upholds the consumer-provider agreement, and by doing so, to moderate intense influence on the user's decision-making, which may be caused by observing filtered feed, compared to what would have been the user's decision-making under the reference feed. With the model introduced above, the auditor's task can be formulated as a hypothesis
testing problem with the following two hypotheses:
\begin{itemize}
    \item \textbf{\emph{The null hypothesis $\calH_0$}}: the auditor (or self-audit) decision is that the platform honors the consumer-provider agreement.
    \item \textbf{\emph{The alternative hypothesis $\calH_1$}}: the auditor (or self-audit) decision is to investigate the platform for a possible violation.
\end{itemize}
\noindent Accordingly, relying on a certain from of data, which we will specify in the sequel, the auditor's detection problem is to determine whether $\calH_0$ or $\calH_1$ is true. We need to specify what kind of ``test" is considered. Given a fixed risk $\delta\in(0,1)$, we expect the auditing procedure to find the true one with probability $1-\delta$, whichever it is. We call such a procedure $\delta$-correct. We consider the following notion of ``frugality", which we name \emph{batch setting}: the auditor specifies in advance the number of samples needed for the test, and announce its decision just after observing the data all at once, and the sample complexity of the test is the smallest sample size of a $\delta$-correct procedure. 

\paragraph{Auditor's data.} 
For $t\geq1$ the auditor observes the filtered and reference feeds $\{\bX^{\s{F}}_{i}(t),\bX^{\s{R}}_{i}(t)\}$, for all (or a subset of) users $u\in\s{U}$, and utilize these to test for regulation violations. There are two ways to access this data without invasions to privacy. First, under self-regulation (currently, almost all platform are entirely self-regulated \cite{klonick2017new}), the platform obviously has access to those feeds, and therefore, there are no privacy issues. The second option is to provide anonymized data to the auditor. Indeed, both the users identities and the meaning behind the features should/can be removed since they do not affect regulation enforcement. Note that de-anonymization is not a real concern here because the anonymized datasets will not be publicly shared anyhow. Moreover, since the auditor only requires the numerical features of the feeds, rather their semantic interpretation, de-anonymization would require unreasonable significant effort that the auditor is not willing to undertake. Thus, with carefully laid out but reasonable measures, the users data would remain private and anonymous. Finally, notice that in principle the filtered and reference feeds need not necessarily correspond to real users and could represent sufficiently representative sample of hypothetical users.

\subsection{Formalizing the auditor's goal}\label{subsec:regulatorgoalapp}
Let $\{\bX^{\s{F}}_{u}(t)\}_{t\geq1}$ and $\{\bX^{\s{R}}_{u}(t)\}_{t\geq1}$ denote the sequences of user $u$'s filtered an reference feeds evolved over time, respectively. As discussed above, the users implicitly form beliefs from their feeds. With enough evidence, the users gain confidence, and then take actions. Accordingly, the corresponding user $u$'s beliefs and actions are denoted by $\{\calB^{\s{F}}_{u,t},\s{br}^{\s{F}}_{u,t}\}_{t\geq1}$ and $\{\calB^{\s{R}}_{u,t},\s{br}^{\s{R}}_{u,t}\}_{t\geq1}$, implied by the filtered and reference feeds, respectively. 

\paragraph{Violation.} We now define the meaning of ``violation" from the auditor's perspective. Let $\s{T}\in\mathbb{N}$ denote the time horizon, which determines how far into the past the auditor scrutinizes the platform's behavior. Let $\s{d}(\cdot\|\cdot):\Omega\times\Omega\to\mathbb{R}_{\geq0}$ be a probability metric between two probability measures defined over $\Omega$. 
Let $\bar{\s{U}}\subseteq\s{U}$ be a certain subset of users (such representative subset of the entire set of users). Then, define the \emph{total action-variability metric} as follows:
\begin{align}
    \mathds{V}_{\s{action}}\triangleq \frac{1}{\s{T}\cdot|\bar{\s{U}}|}\sum_{i\in\bar{\s{U}}}\sum_{t=1}^{\s{T}}\max_{\s{h}_{t,i}}\s{d}\p{\s{br}^{\s{F}}_{i,t}(\s{h}_{t,i})\Big\|\s{br}^{\s{R}}_{i,t}(\s{h}_{t,i})}.
\end{align}
Similarly, define the \emph{total belief-variability metric} as,
\begin{align}
    \mathds{V}_{\s{belief}}\triangleq \frac{1}{\s{T}\cdot|\bar{\s{U}}|}\sum_{i\in\bar{\s{U}}}\sum_{t=1}^{\s{T}}\max_{\s{h}_{t,i}}\s{d}\p{\calB^{\s{F}}_{i,t}(\s{h}_{t,i})\Big\|\calB^{\s{R}}_{i,t}(\s{h}_{t,i})}.
\end{align}
Finally, recall that in Subsection~\ref{subsec:user-platform} we also proposed a statistical model for filtering. Accordingly, as we explain below, it is beneficial to define also the \emph{total filtering-variability metric}:
 \begin{align}
    \mathds{V}_{\s{filter}}\triangleq \frac{1}{\s{T}\cdot|\bar{\s{U}}|}\sum_{i\in\bar{\s{U}}}\sum_{t=1}^{\s{T}}\max_{\s{h}_{t-1,i}}\s{d}\p{\pr^{\s{F}}_{i,t}(\s{h}_{t-1,i})\Big\|\pr^{\s{R}}_{i,t}(\s{h}_{t-1,i})}.
\label{total-filtering-variability-metricapp}
\end{align}
It is useful to note that there is an analytical relationship between the above variabilities. Indeed, viewing $\s{br}^{\s{F}}_{i,t}$ as a result of a probabilistic kernel that is applied on the beliefs, and assuming that the metric $\s{d}$ satisfies the data processing inequality \cite{CoverBook},
it follows that $\mathds{V}_{\s{action}}\leq\mathds{V}_{\s{belief}}\leq\mathds{V}_{\s{filter}}$. Now, from the auditor's perspective, violation could mean that $\mathds{V}_{\s{action}}>\varepsilon>0$, for some $\varepsilon>0$ which governs the regulation strictness, i.e., higher values of $\varepsilon$ indicate greater strictness. Alternatively, violation can also be defined through the belief-variability, namely, $\mathds{V}_{\s{belief}}>\varepsilon>0$, for some $\varepsilon>0$. Accordingly, depending on the auditor's ambition, its testing/decision problem can be formulated as one of the following:
\begin{subequations}\label{eqn:HTbelief}
  \begin{align}
   &\calH_{0}:\mathds{V}_{\s{action}}\leq\varepsilon_1\hspace{0.03cm}\quad\s{vs.}\quad\calH_{1}:\mathds{V}_{\s{action}}\geq\varepsilon_2,\label{eqn:HTbelief1}\\
   &\calH'_{0}:\mathds{V}_{\s{belief}}\leq\varepsilon_1\;\quad\s{vs.}\quad\calH'_{1}:\mathds{V}_{\s{belief}}\geq\varepsilon_2,\label{eqn:HTbelief2}\\
   &\calH''_{0}:\mathds{V}_{\s{filter}}\leq\varepsilon_1\hspace{0.2cm}\quad\s{vs.}\quad\calH''_{1}:\mathds{V}_{\s{filter}}\geq\varepsilon_2,\label{eqn:HTbelief3app}
  \end{align}
\end{subequations}
where $\varepsilon_2>\varepsilon_1\geq0$. Devising successful statistical tests which solve \eqref{eqn:HTbelief1} (or, \eqref{eqn:HTbelief2}) with high probability, guarantee that whenever the auditor decision is $\calH_0$ (or, $\calH_0$'), then the platform honors the consumer-provider agreement, since the beliefs and actions are indistinguishable under the filtered and reference feeds. Conversely, rejecting $\calH_0$ (or, $\calH_0'$) with high confidence implies that AF causes significantly different learning outcomes. Note that by the data processing inequality, accepting $\calH_0''$ in \eqref{eqn:HTbelief3app} imply immediately that $\calH_0$ and $\calH_0'$ hold as well. Note that the general form of the hypothesis testing problems formulated in \eqref{eqn:HTbelief} reminiscent of the well-studied \emph{tolerant closeness testing} problem (see, e.g., \cite{daskalakis2018distribution,canonne2021price}). In this paper, we focus on the hypothesis test in \eqref{eqn:HTbelief3app}. 

\paragraph{Testing.} Solving \eqref{eqn:HTbelief3app} is mathematically intractable unless we place further assumptions on the family of distributions that generate the feeds. In this paper, we assume the following quasi-Markov homogeneous model. We divide the time horizon into batches, and assume that in each batch, the platform filtering process is modeled as a large probabilistic state machine. During these batches the platform collect new data to create new successive feeds. 
From the auditor's point of view, the platform is a rather sequentially-feeds generating system, making a probabilistic relationship of the current feed conditioned on the previous feeds, in time intervals. Under these circumstances, the auditor models problem \eqref{eqn:HTbelief3} as a quasi-Markov homogeneous model. 

Mathematically, let $\s{T}_{\s{total}}\in\mathbb{N}$ denote the time horizon, which determines how far into the past the auditor scrutinizes the platform's behavior. Assume we have $\s{B}\in\mathbb{N}$ batches each of length $\s{T}\in\mathbb{N}$, such that in batch $b\in[\s{B}]$ we have a time sampling sequence $b\cdot\s{T}<t_{0,b}<t_{1,b}<\dots<t_{\s{T},b}\leq (b+1)\cdot\s{T}$. In each batch, from the auditor's point of view, the piece of content $\f{x}_{\ell,u}^{\s{F}}(t_{i,b})$, at time $t_{i,b}$, for $\ell\in[\s{M}]$, is drawn from a first-order irreducible Markov chain, namely, $ \pr(\f{x}^{\s{F}}_{\ell,u}(t_{i,b})\vert\f{x}^{\s{F}}_{\ell,u}(t_{0,b}),\ldots,\f{x}^{\s{F}}_{\ell,u}(t_{i-1,b})) = \pr(\f{x}^{\s{F}}_{\ell,u}(t_{i,b})\vert\f{x}^{\s{F}}_{\ell,u}(t_{i-1,b}))$, and $\pr(\f{x}^{\s{F}}_{\ell,u}(t_{i,b})=s_2\vert\f{x}^{\s{F}}_{\ell,u}(t_{i-1,b})=s_1)\triangleq Q_{u,b}(s_1,s_2)$, for any two possible states $s_1,s_2\in\calX$. We denote the transition probability matrix by in batch $b\in[\s{B}]$ by $\s{Q}^{\f{F}}_{u,b}=\left[Q_{u,b}(s_1,s_2)\right]_{s_1,s_2 \in\calX}$. We assume further that the $\s{M}$ Markov trajectories are i.i.d. Note that over different intervals, indexed by $b$, the filtering process could be transformed into a new state machine subjected to a different transition probabilities. For example, this transformation may occur over time when new external data incur noticeable changes in the platform's reward. Thus, in the $b$th batch, the observed feeds are,
$$\underset{\text{Feed 1}}{\left\{\f{x}^{\s{F}}_{l,u}(t_{0,b})\right\}_{l=1}^{\s{M}}}, \underset{\text{Feed 2}}{\left\{\f{x}^{\s{F}}_{l,u}(t_{1,b})\right\}_{l=1}^{\s{M}}},
\ldots,
\underset{\text{Feed } \s{T}}{\left\{\f{x}^{\s{F}}_{l,u}(t_{\s{T},b})\right\}_{l=1}^{\s{M}}}.$$ 
The above discussion is relevant to the reference feeds generation process as well; in particular, we denote by $\s{P}^{\f{R}}_{u,n}\triangleq\left[P_{u,n}(s_1,s_2)\right]_{i,j \in\calX}$ the corresponding matrix transition probabilities.

From the auditor point of view, in terms of the reward-based platform filtering, a practical interpretation for the above modeling is as follows. At any interval $b$, the platform generates some updated Markovian transition-matrix that is subjected to an updated Markov chain by maximizing its reward function, i.e.,
\begin{equation*}
\begin{aligned}
\displaystyle \s{Q}^{\f{F}}_{u,b+1} = &\argmax_{\s{Q}}\quad \s{Rew}^{\s{F}}_t(\s{Q},\calP_{u,b})\\
& \textrm{ s.t.} \quad \forall j\in\calX, \quad\sum_{i\in\calX} \s{Q}_{i,j}=1, \\\end{aligned}
\end{equation*}
where $\calP_{u,b}$ captures the external data and inputs to the platform used for filtering, intended for user $u$, and was collected during the current time interval $(b\cdot \s{T},(b+1)\cdot \s{T}]$. Similarly, the reference feeds are generated by the same statistical process but by the reference-based rewards objective, i.e.,
\begin{equation*}
\begin{aligned}
\displaystyle \s{P}^{\f{R}}_{u,b+1}= &\argmax_{\s{P}}\quad \s{Rew}^{\s{R}}_t(\s{P},\calP_{u,b})\\
& \textrm{ s.t.} \quad \forall j\in\calX, \quad\sum_{i\in\calX}\s{P}_{i, j}=1. \\\end{aligned}
\end{equation*}

\vskip 0.2in

\acks{This work was supported by the ISRAEL SCIENCE FOUNDATION (grant No. 1734/21).}

\bibliographystyle{alpha}  
\bibliography{sample}  

\newcommand{\etalchar}[1]{$^{#1}$}
\begin{thebibliography}{DWMW17}

\bibitem[ADK15]{acharya2015optimal}
Jayadev Acharya, Constantinos Daskalakis, and Gautam Kamath.
\newblock Optimal testing for properties of distributions.
\newblock {\em Advances in Neural Information Processing Systems}, 28, 2015.

\bibitem[ADLO11]{acemoglu2011bayesian}
Daron Acemoglu, Munther~A Dahleh, Ilan Lobel, and Asuman Ozdaglar.
\newblock Bayesian learning in social networks.
\newblock {\em The Review of Economic Studies}, 78(4):1201--1236, 2011.

\bibitem[AR17]{Anderson17}
J~Anderson and L.~Rainie.
\newblock The future of truth and misinformation online.
\newblock {\em Pew Research Center}, 2017.

\bibitem[Ban92]{banerjee1992simple}
Abhijit~V Banerjee.
\newblock A simple model of herd behavior.
\newblock {\em The quarterly journal of economics}, 107(3):797--817, 1992.

\bibitem[Ber17]{berghel2017lies}
Hal Berghel.
\newblock Lies, damn lies, and fake news.
\newblock {\em Computer}, 50(2):80--85, 2017.

\bibitem[BFR{\etalchar{+}}13]{batu2013testing}
Tu{\u{g}}kan Batu, Lance Fortnow, Ronitt Rubinfeld, Warren~D Smith, and Patrick
  White.
\newblock Testing closeness of discrete distributions.
\newblock {\em Journal of the ACM (JACM)}, 60(1):1--25, 2013.

\bibitem[Bla18]{Blake2018}
Aaron Blake.
\newblock A new study suggests fake news might have won donald trump the 2016
  election.
\newblock 2018.

\bibitem[Boa20]{OversightBoard2021}
Oversight Board.
\newblock Ensuring respect for free expression, through independent judgment.
\newblock March 2020.

\bibitem[Boz13]{bozdag2013bias}
Engin Bozdag.
\newblock Bias in algorithmic filtering and personalization.
\newblock {\em Ethics and information technology}, 15(3):209--227, 2013.

\bibitem[Bra19]{brannon2019free}
Valerie~C Brannon.
\newblock Free speech and the regulation of social media content.
\newblock {\em Congressional Research Service}, 27, 2019.

\bibitem[Cam19]{Campbell19}
A.~Campbell.
\newblock How data privacy laws can fight fake news.
\newblock {\em Just security}, 2019.

\bibitem[Cap18]{Robyn2018AlgorithmicFiltering}
Robyn Caplan.
\newblock 29 algorithmic filtering.
\newblock {\em Mediated Communication}, page 561, 2018.

\bibitem[CBM18]{chau2018use}
Michelle~M Chau, Marissa Burgermaster, and Lena Mamykina.
\newblock The use of social media in nutrition interventions for adolescents
  and young adults—a systematic review.
\newblock {\em International journal of medical informatics}, 120:77--91, 2018.

\bibitem[CC19]{chesney2019deep}
Bobby Chesney and Danielle Citron.
\newblock Deep fakes: A looming challenge for privacy, democracy, and national
  security.
\newblock {\em Calif. L. Rev.}, 107:1753, 2019.

\bibitem[CDL21]{pmlr-v132-chan21a}
Siu~On Chan, Qinghua Ding, and Sing~Hei Li.
\newblock Learning and testing irreducible markov chains via the $k$-cover
  time.
\newblock In Vitaly Feldman, Katrina Ligett, and Sivan Sabato, editors, {\em
  Proceedings of the 32nd International Conference on Algorithmic Learning
  Theory}, volume 132 of {\em Proceedings of Machine Learning Research}, pages
  458--480. PMLR, 16--19 Mar 2021.

\bibitem[CDVV14]{chan2014optimal}
Siu-On Chan, Ilias Diakonikolas, Paul Valiant, and Gregory Valiant.
\newblock Optimal algorithms for testing closeness of discrete distributions.
\newblock In {\em Proceedings of the twenty-fifth annual ACM-SIAM symposium on
  Discrete algorithms}, pages 1193--1203. SIAM, 2014.

\bibitem[Cho17]{chouldechova2017fair}
Alexandra Chouldechova.
\newblock Fair prediction with disparate impact: A study of bias in recidivism
  prediction instruments.
\newblock {\em Big data}, 5(2):153--163, 2017.

\bibitem[CJKL22]{canonne2021price}
Clement~L Canonne, Ayush Jain, Gautam Kamath, and Jerry Li.
\newblock The price of tolerance in distribution testing.
\newblock In {\em Proceedings of Thirty Fifth Conference on Learning Theory},
  volume 178 of {\em Proceedings of Machine Learning Research}, pages 573--624.
  PMLR, 02--05 Jul 2022.

\bibitem[CMS23]{cen2023userdriven}
Sarah~H. Cen, Aleksander Madry, and Devavrat Shah.
\newblock A user-driven framework for regulating and auditing social media.
\newblock {\em arXiv preprint arXiv:2304.10525}, 2023.

\bibitem[CS20]{cen20OLD}
Sarah~H. Cen and Devavrat Shah.
\newblock Regulating algorithmic filtering on social media.
\newblock {\em arXiv preprint: arxiv.org/pdf/2006.09647v3.pdf}, 2020.

\bibitem[CS21]{cen2020regulating}
Sarah Cen and Devavrat Shah.
\newblock Regulating algorithmic filtering on social media.
\newblock In {\em Advances in Neural Information Processing Systems},
  volume~34, pages 6997--7011. Curran Associates, Inc., 2021.

\bibitem[CT06]{CoverBook}
Thomas~M. Cover and Joy~A. Thomas.
\newblock {\em Elements of Information Theory (Wiley Series in
  Telecommunications and Signal Processing)}.
\newblock Wiley-Interscience, USA, 2006.

\bibitem[DCJ{\etalchar{+}}19]{Collins2019Elliott}
Collins Damian, Efford Clive, Elliott Julie, Farrelly Paul, Hart Simon, Knight
  Julian, Ian~C. Lucas, OHara Brendan, Pow Rebecca, Stevens Jo, and Watling
  Giles.
\newblock Disinformation and ‘fake news’.
\newblock 2019.

\bibitem[DDG18]{daskalakis2018testing}
Constantinos Daskalakis, Nishanth Dikkala, and Nick Gravin.
\newblock Testing symmetric markov chains from a single trajectory.
\newblock In {\em Conference On Learning Theory}, pages 385--409. PMLR, 2018.

\bibitem[DGB17]{devito2017algorithms}
Michael~A DeVito, Darren Gergle, and Jeremy Birnholtz.
\newblock " algorithms ruin everything" \# riptwitter, folk theories, and
  resistance to algorithmic change in social media.
\newblock In {\em Proceedings of the 2017 CHI conference on human factors in
  computing systems}, pages 3163--3174, 2017.

\bibitem[DKW18]{daskalakis2018distribution}
Constantinos Daskalakis, Gautam Kamath, and John Wright.
\newblock Which distribution distances are sublinearly testable?
\newblock In {\em Proceedings of the Twenty-Ninth Annual ACM-SIAM Symposium on
  Discrete Algorithms}, pages 2747--2764. SIAM, 2018.

\bibitem[DMNS06]{dwork2006calibrating}
Cynthia Dwork, Frank McSherry, Kobbi Nissim, and Adam Smith.
\newblock Calibrating noise to sensitivity in private data analysis.
\newblock In {\em Theory of cryptography conference}, pages 265--284. Springer,
  2006.

\bibitem[DR{\etalchar{+}}14]{dwork2014algorithmic}
Cynthia Dwork, Aaron Roth, et~al.
\newblock The algorithmic foundations of differential privacy.
\newblock {\em Found. Trends Theor. Comput. Sci.}, 9(3-4):211--407, 2014.

\bibitem[DWMW17]{davidson2017automated}
Thomas Davidson, Dana Warmsley, Michael Macy, and Ingmar Weber.
\newblock Automated hate speech detection and the problem of offensive
  language.
\newblock In {\em Proceedings of the International AAAI Conference on Web and
  Social Media}, volume~11, pages 512--515, 2017.

\bibitem[Gho19]{ghosh2019new}
Dipayan Ghosh.
\newblock A new digital social contract is coming for silicon valley.
\newblock {\em Harvard Business Review}, 27, 2019.

\bibitem[Ind21]{YouTube2021self_regulation}
The Independent.
\newblock Amid capitol violence, facebook, youtube remove trump video.
\newblock January 2021.

\bibitem[JHF{\etalchar{+}}18]{jane2018social}
Monica Jane, Martin Hagger, Jonathan Foster, Suleen Ho, and Sebely Pal.
\newblock Social media for health promotion and weight management: a critical
  debate.
\newblock {\em BMC public health}, 18(1):1--7, 2018.

\bibitem[JO21]{DBLP:journals/corr/abs-2106-00742}
Md~Saroar Jahan and Mourad Oussalah.
\newblock A systematic review of hate speech automatic detection using natural
  language processing.
\newblock {\em CoRR}, abs/2106.00742, 2021.

\bibitem[KBB20]{doi:10.1080/15213269.2019.1598434}
Sanne Kruikemeier, Sophie~C. Boerman, and Nadine Bol.
\newblock Breaching the contract? using social contract theory to explain
  individuals’ online behavior to safeguard privacy.
\newblock {\em Media Psychology}, 23(2):269--292, 2020.

\bibitem[Klo17]{klonick2017new}
Kate Klonick.
\newblock The new governors: The people, rules, and processes governing online
  speech.
\newblock {\em Harv. L. Rev.}, 131:1598, 2017.

\bibitem[Kos14]{koszegi2014behavioral}
Botond Koszegi.
\newblock Behavioral contract theory.
\newblock {\em Journal of Economic Literature}, 52(4):1075--1118, 2014.

\bibitem[Kur16]{kurbalija2016introduction}
Jovan Kurbalija.
\newblock {\em An introduction to internet governance}.
\newblock Diplo Foundation, 2016.

\bibitem[Lee16]{lee2016impact}
Francis~LF Lee.
\newblock Impact of social media on opinion polarization in varying times.
\newblock {\em Communication and the Public}, 1(1):56--71, 2016.

\bibitem[LM17]{marwick2017media}
Becca Lewis and Alice~E. Marwick.
\newblock Media {M}anipulation and {D}isinformation {O}nline.
\newblock {\em New York: Data \& Society Research Institute}, 2017.

\bibitem[LRR11]{ReutDana}
Reut Levi, Dana Ron, and Ronitt Rubinfeld.
\newblock Testing properties of collections of distributions.
\newblock volume~17, pages 179--194, 01 2011.

\bibitem[Man89]{manning1989implicit}
Alan Manning.
\newblock Implicit contract theory.
\newblock {\em Current Issues in Labour Economics}, page~63, 1989.

\bibitem[Med21]{medzini2021enhanced}
Rotem Medzini.
\newblock Enhanced self-regulation: The case of facebook’s content
  governance.
\newblock {\em New Media \& Society}, page 1461444821989352, 2021.

\bibitem[MGBS16]{mitchell2016modern}
Amy Mitchell, Jeffrey Gottfried, Michael Barthel, and Elisa Shearer.
\newblock The modern news consumer: News attitudes and practices in the digital
  era.
\newblock 2016.

\bibitem[MRH19]{Sina19}
Sina Mohseni, Eric~D. Ragan, and Xia Hu.
\newblock Open issues in combating fake news: Interpretability as an
  opportunity.
\newblock {\em arXiv:1711.04024}, 2019.

\bibitem[MTSJ18]{molavi2018theory}
Pooya Molavi, Alireza Tahbaz-Salehi, and Ali Jadbabaie.
\newblock A theory of non-bayesian social learning.
\newblock {\em Econometrica}, 86(2):445--490, 2018.

\bibitem[New21]{Twitter2021self_regulation}
BBC News.
\newblock Twitter suspends 70,000 accounts linked to qanon.
\newblock January 2021.

\bibitem[OW15]{obar2015social}
Jonathan~A Obar and Steven~S Wildman.
\newblock Social media definition and the governance challenge-an introduction
  to the special issue.
\newblock {\em Obar, JA and Wildman, S.(2015). Social media definition and the
  governance challenge: An introduction to the special issue.
  Telecommunications policy}, 39(9):745--750, 2015.

\bibitem[Par11]{Pariser2011}
Eli Pariser.
\newblock How the new personalized web is changing what we read and how we
  think.
\newblock 2011.

\bibitem[Pas19]{Paschen19}
Jeannette Paschen.
\newblock Investigating the emotional appeal of fake news using artificial
  intelligence and human contributions.
\newblock {\em Journal of Product \& Brand Management}, 05 2019.

\bibitem[Pet00]{petty2000marketing}
Ross~D Petty.
\newblock Marketing without consent: Consumer choice and costs, privacy, and
  public policy.
\newblock {\em Journal of Public Policy \& Marketing}, 19(1):42--53, 2000.

\bibitem[Qui16]{quinn2016we}
Kelly Quinn.
\newblock Why we share: A uses and gratifications approach to privacy
  regulation in social media use.
\newblock {\em Journal of Broadcasting \& Electronic Media}, 60(1):61--86,
  2016.

\bibitem[RAC19]{8669073}
Axel Rodríguez, Carlos Argueta, and Yi-Ling Chen.
\newblock Automatic detection of hate speech on facebook using sentiment and
  emotion analysis.
\newblock In {\em 2019 International Conference on Artificial Intelligence in
  Information and Communication (ICAIIC)}, pages 169--174, 2019.

\bibitem[RR20]{racz2020rumor}
Mikl{\'o}s~Z R{\'a}cz and Jacob Richey.
\newblock Rumor source detection with multiple observations under adaptive
  diffusions.
\newblock {\em IEEE Transactions on Network Science and Engineering},
  8(1):2--12, 2020.

\bibitem[RSS23]{rivera_sauerwald_sylvester_2023}
Nicolás Rivera, Thomas Sauerwald, and John Sylvester.
\newblock Multiple random walks on graphs: mixing few to cover many.
\newblock {\em Combinatorics, Probability and Computing}, 32(4):594–637,
  2023.

\bibitem[SAV{\etalchar{+}}18]{speicher2018potential}
Till Speicher, Muhammad Ali, Giridhari Venkatadri, Filipe~Nunes Ribeiro, George
  Arvanitakis, Fabrício Benevenuto, Krishna~P. Gummadi, Patrick Loiseau, and
  Alan Mislove.
\newblock Potential for {D}iscrimination in {O}nline {T}argeted {A}dvertising.
\newblock In Sorelle~A. Friedler and Christo Wilson, editors, {\em Proceedings
  of the 1st Conference on Fairness, Accountability and Transparency},
  volume~81 of {\em Proceedings of Machine Learning Research}, pages 5--19, New
  York, NY, USA, 23--24 Feb 2018. PMLR.

\bibitem[SCP{\etalchar{+}}14]{siersdorfer2014analyzing}
Stefan Siersdorfer, Sergiu Chelaru, Jose~San Pedro, Ismail~Sengor Altingovde,
  and Wolfgang Nejdl.
\newblock Analyzing and mining comments and comment ratings on the social web.
\newblock {\em ACM Transactions on the Web (TWEB)}, 8(3):1--39, 2014.

\bibitem[SW17]{sarikakis2017social}
Katharine Sarikakis and Lisa Winter.
\newblock Social media users’ legal consciousness about privacy.
\newblock {\em Social Media+ Society}, 3(1):2056305117695325, 2017.

\bibitem[Swe13]{sweeney2013discrimination}
Latanya Sweeney.
\newblock Discrimination in online ad delivery: Google ads, black names and
  white names, racial discrimination, and click advertising.
\newblock {\em Queue}, 11(3):10--29, 2013.

\bibitem[WK19]{wolfer2019minimax}
Geoffrey Wolfer and Aryeh Kontorovich.
\newblock Minimax learning of ergodic markov chains.
\newblock In {\em Algorithmic Learning Theory}, pages 904--930. PMLR, 2019.

\bibitem[WK20]{wolfer2020minimax}
Geoffrey Wolfer and Aryeh Kontorovich.
\newblock Minimax testing of identity to a reference ergodic markov chain.
\newblock In {\em International Conference on Artificial Intelligence and
  Statistics}, pages 191--201. PMLR, 2020.

\bibitem[WM19]{wachter2019right}
Sandra Wachter and Brent Mittelstadt.
\newblock A right to reasonable inferences: re-thinking data protection law in
  the age of big data and {AI}.
\newblock {\em Colum. Bus. L. Rev.}, page 494, 2019.

\bibitem[{Wor}16]{agendaworld}
{World Economic Forum}.
\newblock World economic forum global agenda council on the future of software
  and society. {A} call for agile governance principles.
\newblock Technical report, 2016.
\newblock \url
  {https://www3.weforum.org/docs/IP/2016/ICT/Agile_Governance_Summary.pdf}.

\bibitem[WZ10]{wasserman2010statistical}
Larry Wasserman and Shuheng Zhou.
\newblock A statistical framework for differential privacy.
\newblock {\em Journal of the American Statistical Association},
  105(489):375--389, 2010.

\end{thebibliography}



\end{document}